\documentclass{article}
\usepackage[nonatbib,final]{neurips_2024}

\usepackage[utf8]{inputenc} % allow utf-8 input
\usepackage[T1]{fontenc}    % use 8-bit T1 fonts
\usepackage{hyperref}       % hyperlinks
\hypersetup{colorlinks=true,urlcolor=blue,linkcolor=blue,citecolor=blue}

\usepackage{url}            % simple URL typesetting
\usepackage{booktabs}       % professional-quality tables
\usepackage{amsfonts}       % blackboard math symbols
\usepackage{nicefrac}       % compact symbols for 1/2, etc.
\usepackage{microtype}      % microtypography
\usepackage{xcolor}    

\usepackage{amsmath}
\usepackage{graphicx}
\usepackage{subfigure}
\usepackage{amsthm} 
\usepackage{bm}

\usepackage{ifthen}
\usepackage{xcolor}
\usepackage{amsfonts}
\usepackage{algpseudocode}
\usepackage{amssymb}
\usepackage{algorithm}
\usepackage{wrapfig}
\usepackage{arydshln}
\usepackage{enumitem}
\usepackage{bbding}

\usepackage{thm-restate}

\newcommand{\compilehidecomments}{false}
\ifthenelse{ \equal{\compilehidecomments}{true} }{
	\newcommand{\siwei}[1]{}
	\newcommand{\wei}[1]{}
	\newcommand{\yifei}[1]{}
    \newcommand{\shi}[1]{}
    \newcommand{\shanghua}[1]{}
}{
	\newcommand{\siwei}[1]{{\color{red} [\text{Siwei:} #1]}}
	\newcommand{\wei}[1]{{\color{purple} [\text{Wei:} #1]}}
	\newcommand{\yifei}[1]{{\color{teal} [\text{Yifei:} #1]}}
    \newcommand{\shi}[1]{{\color{green} [\text{Shi:} #1]}}
    \newcommand{\shanghua}[1]{{\color{blue} [\text{Shanghua:} #1]}}
}

\usepackage[english]{babel}

\newtheorem{theorem}{Theorem}
\newtheorem{fact}[theorem]{Fact}

\newcommand{\I}{\mathbb{I}}

\title{ALPINE: Unveiling The Planning Capability of Autoregressive Learning in Language Models}

\author{Siwei Wang$^{1}$\thanks{\ denotes equal contributions. Corresponding author: Wei Chen (\texttt{weic@microsoft.com})} \hspace{0.1cm} 
Yifei Shen$^{1*}$\hspace{0.1cm} 
Shi Feng$^{2}$\hspace{0.1cm}  
Haoran Sun$^{3}$\hspace{0.1cm} 
Shang-Hua Teng$^{4}$\thanks{Supported by a Simons Foundation Investigator Award.}\hspace{0.1cm} 
Wei Chen$^{1}$\Envelope
\\
   \normalfont $^1$Microsoft Research Asia (\texttt{\{siweiwang, yifeishen, weic\}@microsoft.com})\\
  $^2$Harvard University (\texttt{shifeng@fas.harvard.edu}) \\
  $^3$Peking University (\texttt{sunhaoran0301@stu.pku.edu.cn}) \\
  $^4$University of Southern California (\texttt{shanghua@usc.edu})
}

\begin{document}

\date{}
\maketitle

\sloppy

\begin{abstract}
Planning 
%constitutes 
is a crucial element of both human intelligence and contemporary large language models (LLMs).
 %, aiming to identify potential limitations in their planning abilities. 
%Focusing on the planning capabilities of LLMs, we explore the Transformer architectures to understand the underlying autoregressive learning mechanisms in this paper. Through both theoretical and experimental analyses, we aim to characterize their expressiveness and identify fundamental limitations in their learning and problem-solving abilities.
%
In this paper, we initiate a theoretical investigation into the 
%development
emergence of planning capabilities in Transformer-based LLMs via their next-word prediction mechanisms. We model planning as a network path-finding task, where the objective is to generate a valid path from a specified source node to a designated target node.
 Our mathematical characterization  shows that Transformer architectures can execute path-finding by embedding the adjacency and reachability matrices within their weights. Furthermore, our theoretical analysis of gradient-based learning dynamics reveals that LLMs can learn both the adjacency  and a limited form of the reachability matrices.
 These theoretical insights are then validated through experiments, which demonstrate that Transformer architectures indeed learn the adjacency and an incomplete reachability matrices, consistent with our theoretical predictions.
When applying our methodology to the real-world planning benchmark Blocksworld, our observations remain consistent. Additionally, our analyses uncover a fundamental limitation of current Transformer architectures in path-finding: \emph{these architectures cannot identify reachability relationships through transitivity, which leads to failures in generating paths when concatenation is required}. These findings provide new insights into how the internal mechanisms of autoregressive learning 
facilitate intelligent planning
%enable planning in networks,
and deepen our understanding 
of how future LLMs might achieve
more advanced and general planning-and-reasoning capabilities across diverse applications.
\end{abstract}

\vspace{-2mm}
\section{Introduction}

\vspace{-2mm}
Large language models (LLMs) such as ChatGPT have impressed many with their powerful capabilities across a wide range of tasks, including language processing, knowledge extraction, reasoning, planning, coding, tool use, and more \cite{Achiam2023GPT4TR}.
%Large language models (LLMs) led by ChatGPT have impressed everyone by its powerful capabilities on multi-faceted tasks from language processing, knowledge extraction, to reasoning, planning, coding, tool use, etc.
%The wide range of 
%The broad spectrum of intelligent capabilities exhibited by LLMs 
%reflects promising signs of 
%shows sparks of 
%artificial general intelligence (AGI)~\cite{bubeck2023Sparks} and catalyzes an AI revolution.
%Individuals and organizations are now striving to develop more powerful and adaptive
%race for stronger 
%AI models towards AGI, while also integrating LLM-based AI models into 
%various aspects of our work and daily lives.
%every aspects of our work and life.  
However, we continue to be intrigued by the underlying mechanisms that fuel the power of LLMs. While all current LLMs are built on the Transformer architecture, which uses autoregressive learning to predict the next word in a language sequence, the overarching question remains:
  %that learns in an autoregressive manner to predict next word in a language sequence, 
 % but why does the Transformer-based autoregressive learning architecture generate such strong  performance in various intelligent tasks? 
% Why does the Transformer-based autoregressive learning architecture produce such exceptional performance across a wide range of intelligent tasks?
% To put it in plain English: 
\begin{center}
    \begin{minipage}{0.73\textwidth}
{\em Why does the process of next-word prediction give rise to intelligence?}
%    {\em Why does next-word prediction generate intelligence?}
\end{minipage}
\end{center}
%{\em Why does next-word prediction generate intelligence?}

There is no definite answer to this question yet, but researchers are approaching the problem from various angles, aiming to characterize the power and limitations of LLMs, as well as to capture their underlying acquisition, abstraction, generalization, adaptation, and reasoning mechanisms.
Recently, the mechanisms of grammar learning, knowledge manipulation, scaling laws, and arithmetic operations have been empirically uncovered \cite{allenzhu2023physics,zhu2023physics,zhu2024physics,allen2023physics,zhang2022unveiling,lee2023teaching}. 
Furthermore, theoretical analyses have been conducted on in-context learning, chain-of-thought, and other forms of reasoning \cite{zhang2023trained,giannou2023looped,feng2024towards,yang2024efficient}. 
Beyond these, LLMs' capability for planning—a fundamental component of human intelligence—has also drawn considerable attention. Planning is involved in nearly every aspect of our daily life, such as organizing a task at work, planning a trip, or seeking a mathematical proof of a theorem. Additionally, task planning plays a pivotal role in state-of-the-art LLM-empowered autonomous agents, such as HuggingGPT \cite{shen2024hugginggpt}, Voyager \cite{wang2023voyager}, and Reflection \cite{shinn2024reflexion}. 
Understanding how LLMs complete a planning task can shed light on how the seemingly low-level statistical task of next-word prediction transforms into a high-level intelligent process. Several prior studies have empirically evaluated the planning capabilities of LLMs, yielding both positive and negative results
%This understanding may serve as a potential pathway to comprehend and explain other intelligent behaviors exhibited by LLMs. 
%
%In the last year, there are 
\cite{momennejad2024evaluating,valmeekam2023planning}.
However, the current results are incomplete and do not fully explain why LLMs can or cannot successfully accomplish specific planning tasks (see Appendix~\ref{sec:relatedwork} for a detailed discussion of related works).

%Furthermore, it is important to unveil how planning capabilities emerge in LLMs (Transformer architecture, next token prediction loss, and trained by gradient descent) and the potential limitations of LLMs' planning compared with human beings. 
Given that planning often involves making sequential selections of next steps % within a multi-step procedure
to achieve a desired goal, it naturally relates to the path-finding task in networks.
For example, autonomous agents (e.g., HuggingGPT~\cite{shen2024hugginggpt}) for scheduling API calls can be likened to finding a call path in the API call graph; a mathematical proof can be seen as a proof path from the axioms to the final theorem~\cite{trinh2024solving}; and a step-by-step solution to a grade-school math problem can be viewed as a path in the dependency graph among the variables~\cite{zhu2024physics21,zhu2024physics22}.
Many previous studies on LLM planning capabilities are related to path-finding. e.g., an LLM planning benchmark called Blocksworld \cite{valmeekam2023planning} can be viewed as path-finding from the initial state of the blocks to the final state in a state transition graph. Furthermore, in neuroscience, planning is often evaluated through path-finding in a maze \cite{whittington2022build}.
Consequently, in this paper, we abstract planning in LLM learning as the following path-finding task: given an underlying directed graph, a Transformer architecture is provided with training data consisting of a collection of paths that specify the source node, the target node, and a path from the source to the target. The task of the language model is then to generate a path for a new source-target pair.
In addition to measuring the performance of the trained model, we examine the internal weighting mechanism and the learning dynamics of the Transformer architecture during the learning and problem-solving process.
This research is part of our broader project, ALPINE (Autoregressive Learning for Planning In NEtworks), which aims to answer the overarching question on the connection between the process of next-word prediction and the emergence of intelligence through the lens of planning.

\noindent{\bf Our Contributions}: Our project initiates a theoretical investigation into the development of planning capabilities in Transformer-based language models 
%via their next-word prediction mechanisms.
by focusing on characterizing their expressiveness and learning dynamics in the path-finding task.
First, in Theorem \ref{thm:expressive}, we present a mathematical construction of a Transformer that encodes both the adjacency and reachability matrices of the network, thereby establishing that Transformer architectures possess the expressive capacity to complete the path-finding task.
Then, in Theorem \ref{Thm_1}, we prove that when applying gradient descent to minimize the cross-entropy loss on the training data, a model based on a simplified Transformer architecture can extract the adjacency matrix and a limited form of the reachability matrix, using them to mimic human-like intelligence in path-finding.
%	in generating the next node that is both adjacent to the current node and reachable to the target node  
Our theoretical analysis further reveals a fundamental limitation of current Transformer architectures: they do not learn certain types of reachability, particularly transitive reachability, resulting in an incomplete ability to reason about future steps when planning.
To validate our theoretical findings, we conduct extensive experiments training Transformers on the path language using autoregressive learning. First, these experiments demonstrate that Transformers achieve high accuracy in the path-finding task (Figure \ref{fig1:five_images}). Second, we show that it is indeed possible to extract both the adjacency and a limited form of the reachability matrices from the Transformers' weights (Figures \ref{Figure_C1},\ref{Figure_C2},\ref{fig3:five_images},\ref{fig4:image6}). Third, we observe a significant drop in test accuracy when the source and target nodes are connected only through concatenated path segments in the training data (Figure \ref{fig4:five_images}).
These findings align with our theoretical analysis, confirming that  \emph{current Transformers have limitations in learning transitive reachability relationships, unlike human intelligence}. Finally, we validate these results on a real-world task planning benchmark, Blocksworld \cite{valmeekam2023planning}, which directly corresponds to the path-finding problem (see Appendix \ref{sec:blocksworld}).

\vspace{-2mm}
\section{Preliminaries}
\label{sec:prelim}
%\subsection{Notations}
\vspace{-2mm}
Throughout this paper, 
we use the following notations for matrices and vectors: $\bm{a}$ and $\bm{A}$ stand for a column vector and a matrix, respectively. 
Notations $\bm{a}_{(i)}$ and $\bm{A}_{(i,j)}$ %, and ${\bf A}_{(i,j,k)}$ 
	denote the $i^{th}$ entry of vector $\bm{a}$ and the $(i,j)^{th}$ entry in 
% the $i$-th row and the $j$-th column of the
matrix $\bm{A}$, %, the entry indexed by $i,j,k$ in tensor ${\bf A}$,  
	respectively. 
We also denote the $i^{th}$ row of matrix $\bm{A}$ by $\bm{A}_{(i,:)}$. %and the transpose of $\bm{A}$ by $\bm{A}^\top$. %The transpose of $\bm{A}$ is denoted by $\bm{A}^\top$.

\vspace{-2mm}
\subsection{Autoregressive Transformer Architecture and Loss Function}\label{sec:structure}

\vspace{-2mm}
In this paper, we adopt the standard GPT architecture \cite{radford2018improving}. We use the following notation for the architecture and loss function in our analysis. Let $N$ denote the sequence length, $d$ the embedding size, $H$ the number of heads, $d_k = d / H$ the embedding size per head, and $M$ the vocabulary size. One key component of the architecture is the attention mechanism, which is formulated as:
\begin{align}\label{eq:attention}
    \text{Attention}(\bm{Q}, \bm{K}, \bm{V}) = \textbf{softmax}\left(\frac{\bm{Q}\bm{K}^\top}{\sqrt{d_k}}  \right)\bm{V}
\end{align}
where $\bm{Q} \in \mathbb{R}^{N \times d_k}$, $\bm{K} \in \mathbb{R}^{N \times d_k}$, $\bm{V} \in \mathbb{R}^{N \times d_k}$ are the query, key, and value matrices, respectively. 
%Function $\textbf{softmax}$ takes a vector $\bm{a}\in \mathbb{R}^m$ and transform it into $\bm{b}\in \mathbb{R}^m$ with $\bm{b}_i = e^{\bm{a}_i}/({\sum_{j=1}^m e^{\bm{a}_j}})$, and when $\textbf{softmax}$ 	applies to a matrix, it applies to every row of the matrix.
Denoting $\bm{X} \in \mathbb{R}^{N \times d}$ as input, the multi-head attention is computed as
%\begin{align}\label{eq:MHA}
    $\text{MHA}(\bm{X}) = \text{Concat}_{i \in [H]} (\text{Attention}(\bm{X}\bm{W}_i^Q, \bm{X}\bm{W}_i^K, \bm{X}\bm{W}_i^V))$,
%\end{align}
where $\bm{W}^Q_i \in \mathbb{R}^{d \times d_k}$, $\bm{W}^K_i \in \mathbb{R}^{d \times d_k}$, $\bm{W}^V_i \in \mathbb{R}^{d \times d_k}$ are the learnable weight matrices for the query, key, and value matrices of the $i$-th head. 

The feed-forward layer is a two-layer multi-layer perceptron (MLP) defined as follows:
\begin{align}\label{eq:FFN}
    \text{FFN}(\bm{X}) = \max(\bm{0}, \bm{X}\bm{W}_1 + \bm{1}_{N\times 1}\bm{b}_1^\top)\bm{W}_2 + \bm{1}_{N\times 1}\bm{b}_2^\top
\end{align}
where $\bm{W}_1 \in \mathbb{R}^{d \times 4d}$, $\bm{W}_2 \in \mathbb{R}^{4d \times d}$, $\bm{b}_1 \in \mathbb{R}^{4d}$, and $\bm{b}_2 \in \mathbb{R}^{d}$ are the learnable weight matrices %of FFN, 
and
	$\bm{1}_{N\times x}$ is the all-one matrix with dimension $N\times x$. 
%\wei{I cannot get the dimension of the matrices to match.}
Finally, one Transformer layer is defined as:
\begin{align}\label{eq:transformer}
    \text{Transformer}(\bm{X}) = \text{FFN}(\text{LN}_2(\text{MHA}(\text{LN}_1(\bm{X})) + \bm{X})) + \text{MHA}(\text{LN}_1(\bm{X})) + \bm{X}
    %\text{Transformer}(\bm{X}) = \text{FFN}(\text{MHA}(\bm{X}) + \bm{X}) + \text{MHA}(\bm{X}) + \bm{X}
\end{align}
where $\text{LN}_1$ and $\text{LN}_2$ are two layer normalizations. 

With these essential components 
%at hand, 
in place, we 
%next
proceed to introduce the procedures of GPT. The training data consists of many sequences of tokens, where each sequence is expressed as $\bm{u} = (u_1, \cdots, u_N)$, in which $u_n$ denotes the token id for the $n$-th token in sequence $\bm{u}$. We first represent the tokens in $\bm{u}$ by a one-hot embedding matrix $\bm{U} \in \mathbb{R}^{N \times M}$, where $\bm{U}_{(n, u_n)} = 1$ and $0$ elsewhere. Then there are learnable token embedding matrix $\bm{W}_t \in \mathbb{R}^{M \times d}$ 
%\wei{Should this matrix be denoted at $\bm{W}_t$ with $t$ standing for token, to match with $\bm{W}_p$ for position embedding?}
and positional embedding matrix $\bm{W}_p \in \mathbb{R}^{N \times d}$, and the input $\bm{H}_0 = \bm{U}\bm{W}_t + \bm{W}_p \in \mathbb{R}^{N \times d}$. 
This input $\bm{H}_0$ is fed into an $L$-layer Transformer, i.e., $\bm{H}_l = \text{Transformer}(\bm{H}_{l-1}) \in \mathbb{R}^{N \times d}$ for $l = 1, \cdots, L$.
%\wei{Need to add layer norm explanation somewhere. Perhaps also need to mention the temperature parameter.}

Finally, the output embedding goes through another layer normalization $\text{LN}_t$, and then it is multiplied by a learnable output weight matrix $\bm{W}_o \in \mathbb{R}^{d \times M}$ to convert back to probability weights over all possible tokens.
We calculate the output probability vector at position $n$, denoted by $\hat{\bm{u}}_{(n+1)}$, using $ \hat{\bm{u}}_{(n+1)} = \textbf{softmax}((\text{LN}_t(\bm{H}_L))_{(n,:)} \bm{W}_o ), 1 \le n < N$.
This probability vector is used
to predict the token for position $n+1$, which reflects the autoregressive learning paradigm.
%\wei{should we give a specific name and notation for $\bm{W}_o$, such as the output weight matrix and $\bm{W}_o$? Later in Section 4 when I read the reference about
%weight matrix $\bm{W}_o$, it sounds weird and I do not know which weight matrix it is talking about, unless I check back at this section to find the definition.}
%\begin{align}\label{eq:lm_head}
%   $ \hat{\bm{u}}_{(n+1)} = \textbf{softmax}((\text{LN}_t(\bm{H}_L))_{(n,:)} \bm{W}_o ), 1 \le n < N$.
%\end{align}
%where $\text{LN}_t$ is another layer-normalization at the end of the Transformer.
%The actual token $u_{k+1}$ is sampled based on the probability vector $\hat{\bm{u}}_{(n+1)}$ and a temperature parameter.
%When temperature parameter is set to $1$, which is what we use throughout the paper, the sample is directly based on the probability value in $\hat{\bm{u}}_{(n+1)}$.

\vspace{-0.1cm}
The adopted loss function is the {\em cross-entropy loss} for the next token prediction, given by:
\begin{align}\label{eq:ce_loss}
    \ell(\bm{u}) %= -\sum_{n} \mathbb{P}\left(\hat{\bm{u}}_{n+1}| u_1, \cdots, u_{n} \right) 
    = -\sum_{n=1}^{N-1} \sum_{k=1}^M \bm{U}_{(n+1,k)} \log \hat{\bm{u}}_{(n+1),k}
\end{align}

%The next token prediction loss is used for unsupervised training. Specifically, given a corpus of tokens $\mathcal{U} = \{u_1, \cdots, u_N\}$, the loss is expressed as
%\begin{align}\label{gpt_loss}
%    \ell(\mathcal{U}) = \sum_{i} \log P(u_i|u_{1}, \cdots, u_{i-1}), 
%\end{align}
%and the probability is given by the neural network
%\begin{align*}
%    \bm{H}_0 = \bm{U}\bm{W}_t  +\bm{W}_p, \quad \bm{H}_l = \text{Transformer}(\bm{H}_0), \quad \bm{x}_{\text{out}} = \text{softmax}(\bm{H}_l \bm{W})
%\end{align*}
%where $\bm{U} = (\bm{u}_1, \cdots, \bm{u}_{-1})$ is the one hot embedding of tokens, $\bm{W}_t \in \mathbb{R}_{d_{\text{model}}}$ is the token embedding, $\bm{W}_p \in \mathbb{R}_{d_{\text{model}}}$ is the positional embedding, and $\bm{W}_o \in \mathbb{R}^{d_{\text{model}} \times \# \text{vocabulary}  }$ is a learnable matrix converting embeddings to tokens.

\subsection{Path-Planning Dataset: Syntax and Data Sources} \label{sec:pathdata}
The dataset is designed to evaluate GPT's path-planning capability on simple graphs.
It is generated from a directed graph $\mathcal{G} = (\mathcal{V},\mathcal{E})$, where $\mathcal{V}$ is the set of nodes, and $\mathcal{E}$ is the set of edges.
For any two nodes $u,v \in \mathcal{V}$, $(u,v) \in \mathcal{E}$ means that there is a directed edge from  $u$ to  $v$ in $\mathcal{G}$. 
A pair of source node $s$ and target node $t$ is considered a {\em valid pair} if 
%at least one path exists
 $\mathcal{G}$ contains at least one path from $s$ to $t$. 
%We allocate a  portion of valid $(s, t)$ pairs to the training dataset and assign the remaining pairs to the test dataset.

The training dataset $\mathcal{D}$ contains sequences in the format ``$s$ $t$ $s$ $a$ $b$ $c$ $t$ $\backslash$n'', 
where $s$ represents the source node token, $t$ the target node token,
%where $s$ is the token for source node, $t$ is the token for target node,  
$s$ $a$ $b$ $c$ $t$ are tokens for nodes in a valid path from $s$ to $t$, and $\backslash$n 
indicates the end of the sequence.
%is the token indicating the end of sequence. 
During testing,  we provide 
	valid pairs of source and target nodes in the format ``$s$ $t$''. The model is tasked with completing the remaining tokens in the sequence.
The completion is  considered correct if the model 
generates a valid path with the correct syntax.

\section{Expressiveness and Learning Dynamics of Transformer Models}
\label{sec:overview}

\subsection{Expressiveness}

%\wei{In Algorithm~\ref{alg:gt}, can we replace Require and Ensure with Input and Output?}

%\wei{Perhaps it is better to mention the adjacency matrix here, since we used it in the next section. I update the algorithm to use the adjacency matrix, and
%revise the paragraph below to mention it.}

%In this section, we present an overview of our results. 

\begin{wrapfigure}{r}{0.55\textwidth}
\vspace{-0.4cm}
  \begin{minipage}{0.54\textwidth}
    \begin{algorithm}[H]  
\caption{A handcrafted path-finding algorithm}
\label{alg:gt}
\begin{algorithmic}[1]  
\State \textbf{Input:} Adjacency matrix $\bm{A}$, reachability matrix $\bm{R}$, source node $s$, target node $t$ 
%\State \textbf{Output:} A list contains a path from $s$ to $t$
\State Set path $P = [s\ t\ s]$ and set current node $i = s$
\While{$i \ne t$}  
    \State Obtain $S = \{k | \bm{A}_{(i,k)}=1 \text{ and }\bm{R}_{(t,k)}=1\}$
    \State Randomly sample next node $k$ from $S$
    \State Append $k$ to path $P$, and set $i = k$
\EndWhile  
\State \textbf{output} path $P$
\end{algorithmic}  
\end{algorithm} 
  \end{minipage}
\end{wrapfigure}
In our path-finding task, the essential step for completing a path is to predict the next node based on the current information. It is evident that to predict the subsequent node on the path, only information related to the current node and the target node is necessary. 
Algorithm \ref{alg:gt} introduces a %handcrafted 
idealized method that utilizes both the adjacency and  reachability matrices of the graph.
%, and
%	it takes $O(|\mathcal{V}|^3)$ offline preparation for the two matrices, and $O(e_{\max} L)$ for each online query, where $e_{\max}$ is the maximum degree of the graph and $L$ is the maximum length among paths from $s$ to $t$. 
The true adjacency matrix $\bm{A}^{\text{\rm true}}$ and the true reachability matrix $\bm{R}^{\text{\rm true}}$ are defined as: %, i.e., 
\begin{align*}
    \bm{A}^{\text{\rm true}}_{(i,k)} =  
\begin{cases}  
  1, & \text{if } (i,k) \in \mathcal{E}, \\
  0, & \text{otherwise.}  
\end{cases}\qquad
    \bm{R}^{\text{\rm true}}_{(t,k)} =  
\begin{cases}  
  1, & \text{if } k \text{ can reach } t \text{ in } \mathcal{G}, \\
  0, & \text{otherwise.}  
\end{cases}
\end{align*}
%where $\bm{A}^{\text{true}}_{(i,k)}=1$ if $(i,k)\in \mathcal{E}$ and $0$ otherwise.
%The true reachability matrix is defined as:
%\wei{Please double check the following definition, and to see if we want $u$ can reach $v$ or $v$ can reach $u$.}
%\begin{align*}
%    \bm{R}^{\text{\rm true}}_{(t,k)} =  
%\begin{cases}  
%  1, & \text{if } k \text{ can reach } t \text{ in } \mathcal{G}, \\
%  0, & \text{otherwise.}  
%\end{cases}
%\end{align*}

\begin{fact}
   Assuming that $t$ is reachable from $s$, then Algorithm \ref{alg:gt} is guaranteed to output a valid path with input $\bm{A} = \bm{A}^{\text{\rm true}}$ and $\bm{R} = \bm{R}^{\text{\rm true}}$.
\end{fact}
%To illustrate the expressive capacities of the Transformer model, we first show that we can manually construct a Transformer to perform the path-finding task by simulating the idealized Algorithm
To illustrate the expressive capacities of the Transformer model, we first demonstrate how to manually construct a Transformer that can perform the path-finding task by simulating 
%the idealized 
Algorithm \ref{alg:gt}.  
%In the manual construction, the task for the model is to find a path from a start node $s$ to a target node $t$ with the format ``$s\ t\ s\ u_1\ u_2\ \ldots\ u_p\ t$''. % where $s\rightarrow u_1\rightarrow u_2\rightarrow\cdots\rightarrow u_p\rightarrow t$ is a path from $s$ to $t$. 
%Consider every time that the Transformer takes ``$s\ t\ s\ u_1\ u_2\ \ldots\ u_k$'' as the input and outputs $u_{k+1}$ for $k=0,1,\ldots,p$ (assuming $s=u_0$ and $t=u_{p+1}$): if $u_{k+1}$ is an out-neighbor of $u_k$ and can reach $t$ in $\mathcal{G}$, we say that the Transformer outputs a correct response.

%\wei{The remaining of the section needs to be revised. It also needs to address the adjacency matrix besides the reachability matrix.}

%The reachability can be obtained by training sequences and we can exactly recover $\bm{R}^{\text{true}}$ if every edge is revealed in training data. 

\begin{restatable}{theorem}{THMEXPRE}\label{thm:expressive}
Given a graph $\mathcal{G}$ (with adjacency matrix $\bm{A}^{\text{\rm true}}$ and reachability matrix $\bm{R}^{\text{\rm true}}$), for every $\varepsilon>0$, there exists a $1$-layer, $1$-head, and $O(|\mathcal{V}|)$-embedding-size Transformer model that generates a valid path for every valid source-target pair $(s,t)$ with probability at least $1-\varepsilon$. % with $\bm{R} = \bm{R}^{\text{true}}$ and $\bm{A} = \bm{A}^{\text{true}}$.
\end{restatable}

The proof involves encoding the adjacency and reachability matrices into the weights of the FFN and attention layers, respectively, while mimicking the computation of Algorithm ~\ref{alg:gt} (see Appendix~\ref{app:construction}).

%\begin{proof}[Proof Sketch.]
%In essence, we 
%%employ 
%utilize the attention layer to attend the output 
%%\emph{only} 
%\emph{solely} 
%to the target node $t$. 
%%By this way,
%This approach allows the distribution of next token $u_{k+1}$ to become a function of both the current node $u_{k}$ and the target node $t$ (as formulated in Section \ref{sec:prelim}). 
%    Then, by integrating the adjacency matrix $\bm{A}^{\text{\rm true}}$ into the MLP layer and the reachability matrix $\bm{R}^{\text{\rm true}}$ into the matrix ${\bm W}^V$ in the attention layer, we extract row vectors 
%    ${\bm R}^{\text{\rm true}}_{(t,:)}$ and ${\bm A}^{\text{\rm true}}_{(u_k,:)}$ from ${\bm R}^{\text{\rm true}}$ and ${\bm A}^{\text{\rm true}}$, respectively, corresponding to the target node $t$ and current node $u_k$. 
%    %Specifically, ${\bm R}^{\text{\rm true}}_{(t,:)}$ and ${\bm A}^{\text{\rm true}}_{(u_k,:)}$ are stored by $\text{MHA}({\bm H}_0)$ and $\text{FFN}(\text{MHA}({\bm H}_0)+{\bm H}_0)$, respectively. 
%    By selecting proper coefficients, we can %ignore the effect of the remaining term ${\bm H}_0$ in $\text{Transformer}({\bm H}_0)$ and only 
%    let the output be the sum of ${\bm R}^{\text{\rm true}}_{(t,:)}$ and ${\bm A}^{\text{\rm true}}_{(u_k,:)}$. 
%    Following the softmax layer, the non-negligible entries in the final vector correspond to the feasible next nodes.
%    With this encoding, the Transformer serves as a simulator of Algorithm~\ref{alg:gt} with input $\bm{A} = \bm{A}^{\text{\rm true}}$ and $\bm{R} = \bm{R}^{\text{\rm true}}$.
%\end{proof}
\vspace{-0.15cm}
\subsection{Learning Dynamics}
\label{sec:dynamics}
Having established the mathematical existence of a Transformer model capable of accomplishing path-finding in a given network, as demonstrated in Theorem \ref{thm:expressive}, we now shift our focus to the following fundamental question:
%\begin{quote}
{\em Can the Transformer architecture, trained on sufficient path data 
	with an autoregressive loss as in Equation \eqref{eq:ce_loss} and using the gradient descent (GD) method, 
	learn the adjacency and reachability matrices and carry out path-finding similar to the idealized Algorithm~\ref{alg:gt}? }
%\end{quote}
%Our main investigation through both theoretical analysis and empirical evaluation presented in the following section is aimed to answer the above question.

Theoretically, we notice that the Transformer may not be able to learn the true adjacency and reachability matrices for the underlying graph.
Instead, it can only learn the relevant information that 
%are 
is directly encoded in the observed training data $\mathcal{D}$.
%Thus
Therefore, we define the
%following
{\em observed} adjacency and reachability matrices based on the training data $\mathcal{D}$ as follows.
%First,
%we notice
%it is important to note that the Transformer may not be capable to learn the exact true adjacency and reachability matrices of the underlying graph.
%and
%\wei{Need to use the real reachability matrix and observed reachability matrix terminology.}
\vspace{-0.03cm}
\begin{align*}
   \bm{A}^{\text{obs}}_{(i,k)}(\mathcal{D}) &=  
\begin{cases}  
  1, & \text{if } \exists \bm{u} \in \mathcal{D}, n \in [3,N-1] \text{ s.t. } u_n = i, u_{n+1} = k\\
  0, & \text{otherwise}  
\end{cases} \\
%\end{align*}
%
%\begin{align*}
   \bm{R}^{\text{obs}}_{(t,k)}(\mathcal{D}) &=  
\begin{cases}  
  1, & \text{if } \exists \bm{u} \in \mathcal{D}, n \in [4,N] \text{ s.t. } u_2 = t, u_n = k\\
  0, & \text{otherwise}.  
\end{cases} 
\end{align*}
\vspace{-0.2cm}

Naturally, the observed adjacency matrix $\bm{A}^{\text{obs}}(\mathcal{D})$ only records the edges $(i,k)$ that appear in some path 
%in
within the training data $\mathcal{D}$. %which is reasonable.
On the other hand, the observed reachability matrix $\bm{R}^{\text{obs}}(\mathcal{D})$ 
%has
exhibits more nuanced distinctions from
%more subtle differences with 
the true reachability matrix.
It only records that node $t$ is reachable from node $k$, if
the training data $\mathcal{D}$ 
	contains a path (sequence) whose target node is $t$ and $k$ appears as a non-source node on the path.
We call such pairs $(t,k)$ \emph{observed reachable pairs}.
Noticeably, reachability through transitivity, i.e., through concatenation of path segments in $\mathcal{D}$, is not observed.
%Therefore, the observed reachability matrix would miss the following types of reachable pairs $(t,k)$ in $\mathcal{G}$ (referred as non-observed reachable pairs): 
%(i) there is no path in $\mathcal{D}$ that contains a sub-path from $k$ to $t$, even if a path from $k$ to $t$ can be obtained by concatenating several sub-paths appeared in $\mathcal{D}$; 
%(ii) there are some paths in $\mathcal{D}$ that contains a sub-path from $k$ to $t$, however, $t$ is not the target node in these paths; 
%(iii) there are some paths in $\mathcal{D}$ that contains a sub-path from $k$ to $t$ and $t$ is the target node in these paths, however, $k$ is always the source node of these paths.

%In Section~\ref{sec:theoretical}, we show that indeed in a simplified Transformer model, the Transformer can only learn the observed adjacency and reachability matrices, not the true underlying matrices.
%The following is an informal version of the result:

%\begin{theorem}\label{thm:gd}[Informal version] By using auto-regressive loss and training with gradient descent, a simplified Transformer architecture with $1$-layer, $1$-head, and $O(|\mathcal{V}|)$-embedding-size simulates Algorithm \ref{alg:gt} with 
%	$\bm{R} = \bm{R}^{\text{obs}}(\mathcal{D})$ and $\bm{A} = \bm{A}^{\text{obs}}(\mathcal{D})$.
%\end{theorem}

%To simplify the analysis,
Here we consider the following simplified 1-layer and 1-head Transformer structure: a) The attention weight is only on the target node (the second token), i.e., we manually set every row in $\textbf{softmax}\left(\frac{\bm{Q}\bm{K}^\top}{\sqrt{d_k}}  \right)$ in Eq. \eqref{eq:attention} to be a one-hot vector with the second coordinate being $1$ (this %corresponds to the perfect attention case, and 
is validated in our experiments shown in Figure \ref{fig2:five_images}), and set the positional embedding matrix $\bm{W}_p = \bm{0}$; b) We remove all the layer normalizations, and use
    %\begin{align}\label{eq:FFN}
    $\text{FFN}(\bm{X}) =  \bm{X} \bm{W}^{M}$ instead of Eq. \eqref{eq:FFN},
    %\begin{align}\label{eq:transformer}
    $\text{Transformer}(\bm{X}) = \text{FFN}(\bm{X}) + \text{MHA}(\bm{X})$ instead of Eq. \eqref{eq:transformer}; c) The token embedding matrix $\bm{W}_t$ and the output weight matrix $\bm{W}_o$ are set to be identity.
The embedding size is the same as the vocabulary size ($d = M$), and we only consider the cross-entropy loss of predicting the next node. %i.e., only when $n \ge 3$ (hence it is not repeating the source node) and the token $u_{n}$ is not the target node (hence it is not predicting ``$\backslash$n").%gives the direction of changes of the parameters in the learnable matrices of the simplified Transformer when following the 
Since there is only one layer and one head, we use $\bm{W}^V$ to represent the weight of the value matrix in the attention layer. 
Under the above Transformer structure, \begin{equation*}
    (\bm{H}_L)_{(n,:)} \bm{W}_o = (\bm{U}\bm{W}_{t} \bm{W}^{M} + \bm{\alpha} \bm{U}\bm{W}_{t}\bm{W}^{V})_{(n,:)}\bm{W}_o = (\bm{U}\bm{W}^{M} + \bm{\alpha}\bm{U}\bm{W}^{V})_{(n,:)} = \bm{W}^{M}_{(u_n,:)} + \bm{W}^{V}_{(u_2,:)},
\end{equation*}
where $\bm{\alpha}$ is the manually set attention weight matrix (every row is a one-hot vector with the second coordinate being $1$).
Therefore, the probability vector when predicting the $(n+1)$-th token is $\textbf{softmax}(\bm{W}^{M}_{(u_n,:)} + \bm{W}^{V}_{(u_2,:)})$, and 
	the prediction probability $\hat{\bm{u}}_{n+1,k}$ equals 
\begin{equation} \label{eq:predweight}
	\hat{\bm{u}}_{n+1,k} = \frac{\exp(\bm{W}^{M}_{(u_n,k)}+\bm{W}^{V}_{(u_2,k)})}{\sum_{\ell} \exp(\bm{W}^{M}_{(u_n,\ell)}+\bm{W}^{V}_{(u_2,\ell)})}.
\end{equation}

%Let the weight of the MLP layer be $M$ and let the weight of the value-layer be $V$. Both of them are $n * n$ matrix, where $n$ is the number of nodes. 

Let $N_{i,j,k}$ be the number of times in $\mathcal{D}$ that: a) the current node is $i$; b) the target node is $j$; c) the next node is $k$, and let $N_{i,j} = \sum_k N_{i,j,k}$, then we have the following theorem. 

\begin{restatable}{theorem}{THMGD}\label{Thm_1}
Under the cross-entropy loss $\ell(\mathcal{D})$, for all $(i,k)$ pairs, i) if $\sum_{j} N_{i,j} = 0$, then ${ \partial \ell(\mathcal{D}) \over \partial \bm{W}^{M}_{(i,k)}}$ is always 0; 
%\wei{for every $i$ and $k$?}
ii) if $\sum_{j} N_{i,j} > 0$ but $\sum_j N_{i,j,k} = 0$, then ${ \partial \ell(\mathcal{D}) \over \partial \bm{W}^{M}_{(i,k)}}$ is always positive; 
%\wei{for every $i$ and $k$?}
iii) if $\sum_j N_{i,j,k} > 0$, then ${ \partial \ell(\mathcal{D}) \over \partial \bm{W}^{M}_{(i,k)}}$ is negative when $\bm{W}^{M}_{(i,k)} \to -\infty$. 
%\wei{Small enough is not rigorous to be put in a theorem. Can we give a more precisely statement?}\siwei{I change it to "converges to $-\infty$".}
%
Similarly, for all $(j,k)$ pairs, i) if $\sum_{i} N_{i,j} = 0$, then ${ \partial \ell(\mathcal{D}) \over \partial \bm{W}^{V}_{(j,k)}}$ is always 0; ii) if $\sum_{i} N_{i,j} > 0$ but $\sum_i N_{i,j,k} = 0$, then ${ \partial \ell(\mathcal{D}) \over \partial \bm{W}^{V}_{(j,k)}}$ is always positive; iii) if $\sum_i N_{i,j,k} > 0$, then ${ \partial \ell(\mathcal{D}) \over \partial \bm{W}^{V}_{(j,k)}}$ is negative when $\bm{W}^{V}_{(j,k)} \to -\infty$. 
\end{restatable}

\vspace{-0.15cm}
\begin{proof}[Proof Sketch.]
By the definition of the cross-entropy loss in Eq.\eqref{eq:ce_loss}, and the prediction probability vector in Eq.\eqref{eq:predweight},
	the total cross-entropy loss of the model (with matrices $\bm{W}^{M}$, $\bm{W}^{V}$) is 
%\wei{It is not very straightforward to go from Eq.\eqref{eq:ce_loss} to the first equation below. Perhaps some intermediate steps would be helpful to show the derivation, in particular,
%	the derivation of $\hat{\bm{u}}_{(n+1),j}$.}
\begin{eqnarray*}
    \ell(\mathcal{D}) = - \sum_{i,j,k} N_{i,j,k} (\bm{W}^{M}_{(i,k)} + \bm{W}^{V}_{(j,k)}) + \sum_{i,j} N_{i,j} \log \left(\sum_{\ell} \exp(\bm{W}^{M}_{(i,\ell)}+\bm{W}^{V}_{(j,\ell)})\right).
\end{eqnarray*}
\vspace{-0.15cm}

%For the cross-entropy loss $\ell(\mathcal{D})$, after some basic calculation, w
Then we can get that: (the proof for the $\bm{W}^{V}$ part is similar)
\begin{equation}\label{Eq_1}
    {\partial \ell(\mathcal{D}) \over \partial \bm{W}^{M}_{(i,k)}} = - \sum_j N_{i,j,k} + 
    \sum_{j} N_{i,j}{\exp(\bm{W}^{M}_{(i,k)}+\bm{W}^{V}_{(j,k)}) \over \sum_{\ell} \exp(\bm{W}^{M}_{(i,\ell)}+\bm{W}^{V}_{(j,\ell)})}.
    % = \sum_j \left(N_{i,j} {\exp(M_{(i,k)}+V_{(j,k)}) \over \sum_{\ell} \exp(M_{i,\ell}+V_{j,\ell})} - N_{i,j,k}\right)
\end{equation}
\vspace{-0.15cm}

In case i), $\sum_{j} N_{i,j} = 0$ implies that $\sum_j N_{i,j,k} = 0$. Hence ${\partial \ell(\mathcal{D}) \over \partial \bm{W}^{M}_{(i,k)}}$ is always 0.

In case ii), $\sum_{j} N_{i,j} > 0$ implies that the second term in Eq. \eqref{Eq_1} is positive, while $\sum_j N_{i,j,k} = 0$ implies that the first term in Eq. \eqref{Eq_1} is 0. Hence ${\partial \ell(\mathcal{D}) \over \partial \bm{W}^{M}_{(i,k)}}$ is always positive.

In case iii), when $\sum_{j} N_{i,j} > 0$ and $\bm{W}^{M}_{(i,k)} \to -\infty$, the second term in Eq. \eqref{Eq_1} converges to 0, and it is smaller than $\sum_j N_{i,j,k} > 0$. Hence, ${ \partial \ell(\mathcal{D}) \over \partial \bm{W}^{M}_{(i,k)}}$ is negative when $\bm{W}^{M}_{(i,k)} \to -\infty$.
\end{proof}

%From the above theorem, we have the following observations. 

%This directly leads to a theoretical explanation on how the model learns the adjacency and reachability:%, as explained below.

\begin{figure}[t]
	\centering
	% Top row with three images
	\begin{subfigure}[True adjacency $\bm{A}^{\text{true}}$% = \bm{A}^{\text{obs}}(\mathcal{D}_1) = \bm{A}^{\text{obs}}(\mathcal{D}_2) = \bm{A}^{\text{obs}}(\mathcal{D}_3)$ 
 ]{
			\centering
			\includegraphics[width=0.223\linewidth]{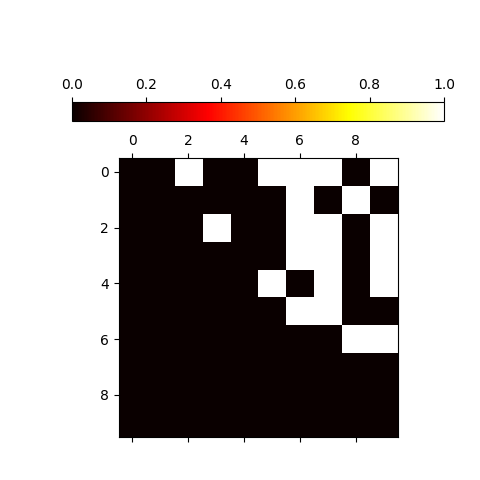}
			\label{Figure_C1_a}}
	\end{subfigure}
	\begin{subfigure}[$\bm{W}^{M}$ under $\mathcal{D}_1$]{
			\centering
			\includegraphics[width=0.223\linewidth]{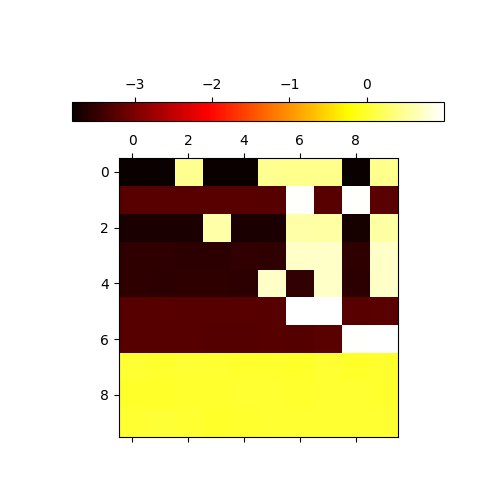}
			\label{Figure_C1_b}}
	\end{subfigure}
	\begin{subfigure}[$\bm{W}^{M}$ under $\mathcal{D}_2$]{
			\centering
			\includegraphics[width=0.223\linewidth]{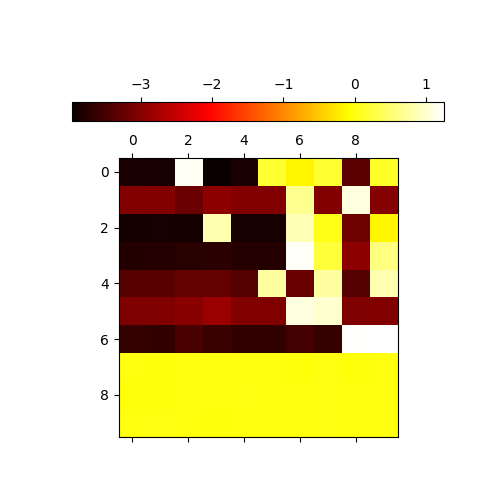}
			\label{Figure_C1_c}}
	\end{subfigure}
	\begin{subfigure}[$\bm{W}^{M}$ under $\mathcal{D}_3$]{
			\centering
			\includegraphics[width=0.223\linewidth]{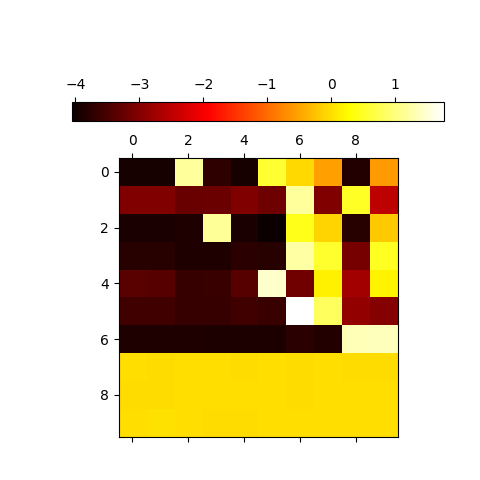}
			\label{Figure_C1_d}}
	\end{subfigure}
	
	\caption{Empirical verification regarding the learning of the adjacency matrix.}
	\label{Figure_C1}
\end{figure}

\vspace{-2mm}
The above technical theorem directly leads to a theoretical explanation on how the model learns the adjacency and reachability information, as explained below.

\vspace{-2mm}
\paragraph{Learning the adjacency matrix.} 
Let $\mathcal{E}(\mathcal{D})$ denote the set of edges appearing in the training dataset $\mathcal{D}$, which corresponds to the observed adjacency matrix $\bm{A}^{\text{obs}}(\mathcal{D})$.
%If all the edges appear at least once in the training dataset $\mathcal{D}$, i.e., 
For any $(i,k) \in \mathcal{E}(\mathcal{D})$, $\sum_{j} N_{i,j,k} > 0$, and for any $(i',k') \notin \mathcal{E}(\mathcal{D})$, $\sum_{j} N_{i',j,k'} = 0$. Then from the above theorem, under the gradient descent learning paradigm, $\bm{W}^{M}_{(i',k')}$ will keep decreasing (since its gradient is always positive), while $\bm{W}^{M}_{(i,k)}$ will not (since its gradient becomes negative when its value is sufficiently negative). This tends to make $\bm{W}^{M}_{(i,k)}$ higher than $\bm{W}^{M}_{(i',k')}$ after training.
%Note that these terms are weights when predicting the next node: a higher $\bm{W}^{M}_{(i,k)}$ means that ``the edge $(i,k)$ exists'', and a lower $\bm{W}^{M}_{(i',k')}$ means that ``the edge $(i',k')$ does not exist''. 
In this way, the Transformer model {\em learns the information about the observed adjacency matrix} with matrix $\bm{W}^{M}$.

To facilitate comprehension, we 
%do 
conducted a simple experiment on the simplified Transformer, and 
%show 
present the results in Figure \ref{Figure_C1}, 
%(the structure of the Transformer
%aligns with the description for Theorem \ref{Thm_1}).
In this experiment, we generate a 10-node graph, and use 3 different training datasets $\mathcal{D}_1, \mathcal{D}_2, \mathcal{D}_3$ based on this graph. $\mathcal{D}_1$ contains all the paths with length 1; $\mathcal{D}_2$ contains all the paths with length 1 and $20\%$ of the paths with length higher than 1; and $\mathcal{D}_3$ contains all the possible paths. 
Figure \ref{Figure_C1_a} is the true adjacency matrix of the graph, which is also the observed adjacency matrix for the three datasets.
Figure \ref{Figure_C1_b}, \ref{Figure_C1_c}, \ref{Figure_C1_d} are the $\bm{W}^{M}$ matrices with training datasets $\mathcal{D}_1$, $\mathcal{D}_2$, $\mathcal{D}_3$, respectively.
%, Figure \ref{Figure_C1_c} is the $\bm{W}^{M}$ matrix with the training dataset $\mathcal{D}_2$, and Figure \ref{Figure_C1_d} is the $\bm{W}^{M}$ matrix with the training dataset $\mathcal{D}_3$.
%\footnote{Matrix $\bm{W}^{M}$ also contains rows and columns
	%corresponding to non-node tokens such as `$\backslash$n', and we remove these rows and columns in the comparison.}
	%Later when we compare empirical matrices $\bm{W}^V$ and $\bm{W}_t\bm{W}_1\bm{W}_2\bm{W}_o$ against theoretical ones, we treat them in the same way.}.
Upon observation, it becomes evident that
these $\bm{W}^{M}$ matrices all successfully capture
%learn 
the structural information 
%of
from the adjacency matrix. 
%Specifically, 
%in the $i^{th}$ row
%of each of these weight matrices, 
%the $k^{th}$ term
%%with 
%corresponding to edge $(i,k) \in \mathcal{E}$ is much higher than the $k^{th}$ term %with 
%corresponding to non-edge $(i,k) \notin \mathcal{E}$.

%As we can see, these $\bm{W}^{M}$'s all learn the information of
%	the adjacency matrix successfully, i.e., in the $i$-th row, the $k$-th terms with $(i,k) \in \mathcal{E}$ are much higher than the $k$-th terms with $(i,k) \notin \mathcal{E}$.

%\begin{corollary}\label{Cor_1}
%    If all the edges appear at least once in the training dataset $\mathcal{D}$, then under a gradient descent training process, the matrix $\bm{W}^{M}$ records the information of the adjacency matrix.
%    \wei{Is ``recording the information of the adjacency matrix'' accurate enough? Can we say it will eventually converge to the adjacency matrix? Probably not, since we do not know if those non-edge entries 
%    	will coverge to zero or not? Also, we will probably change the expression ``record'' or ``memorize'' the matrix with other expressions, such as ``learn'' or ``extract'' the matrix?}
%\end{corollary}

\begin{figure}[t]
	\centering
	\begin{minipage}[c]{0.28\linewidth}
		\begin{subfigure}[True reachability $\bm{R}^{\text{true}}$]{
				\includegraphics[width=0.95\linewidth]{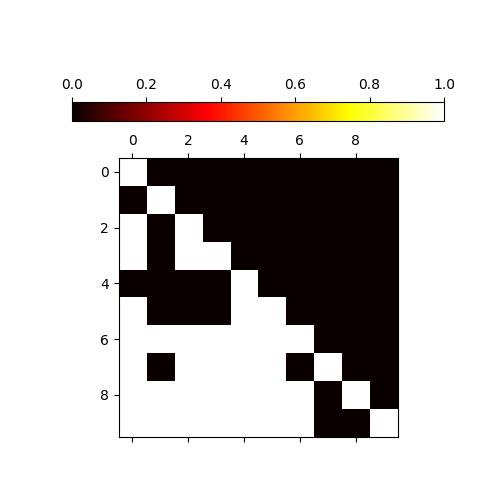}\label{Figure_C2_a}}
		\end{subfigure}
	\end{minipage}
	% 右边的6张图片，3行2列
	\begin{minipage}[c]{0.67\linewidth}
		\begin{subfigure}[%The observed reachability in $\mathcal{D}_1$, 
  $\bm{R}^{\text{obs}}(\mathcal{D}_1)$]{
				\includegraphics[width=0.31\linewidth]{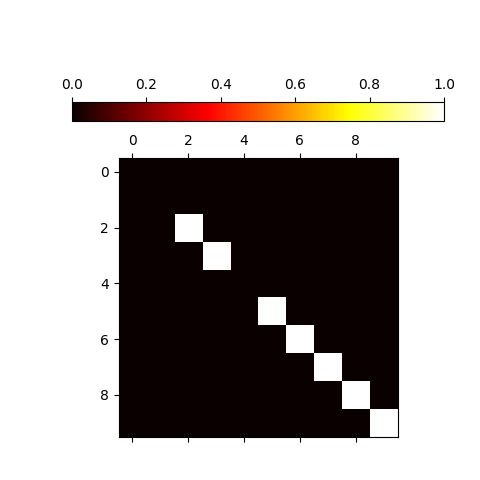}\label{Figure_C2_b}}
		\end{subfigure}
		\begin{subfigure}[%The observed reachability in $\mathcal{D}_2$, 
  $\bm{R}^{\text{obs}}(\mathcal{D}_2)$]{
				\includegraphics[width=0.31\linewidth]{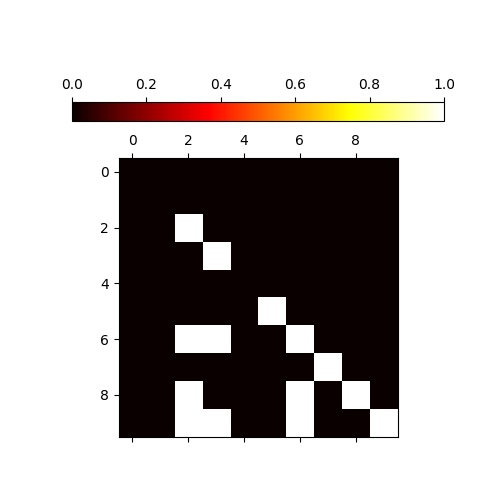}\label{Figure_C2_d}}
		\end{subfigure}
		\begin{subfigure}[%The observed reachability in $\mathcal{D}_3$, 
  $\bm{R}^{\text{obs}}(\mathcal{D}_3)$]{
				\includegraphics[width=0.31\linewidth]{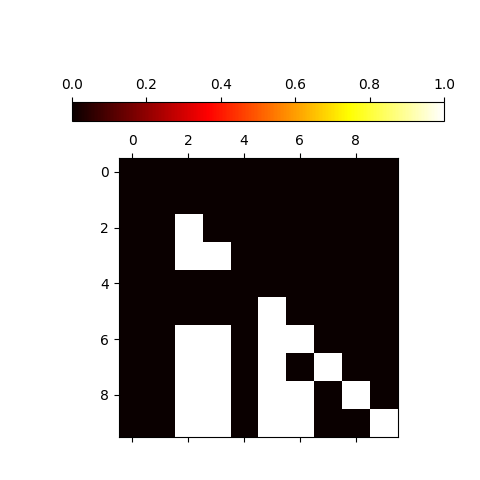}\label{Figure_C2_f}}
		\end{subfigure}
        \begin{subfigure}[%Matrix 
        $\bm{W}^{V}$ under $\mathcal{D}_1$]{
				\includegraphics[width=0.31\linewidth]{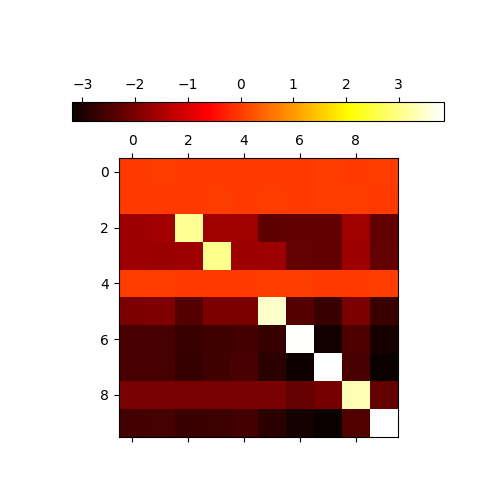}\label{Figure_C2_c}}
		\end{subfigure}
        \begin{subfigure}[%Matrix 
        $\bm{W}^{V}$ under $\mathcal{D}_2$]{
				\includegraphics[width=0.31\linewidth]{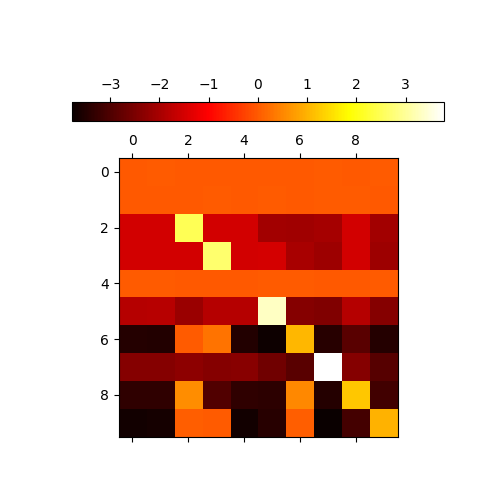}\label{Figure_C2_e}}
		\end{subfigure}
		\begin{subfigure}[%Matrix 
  $\bm{W}^{V}$ under $\mathcal{D}_3$]{
				\includegraphics[width=0.31\linewidth]{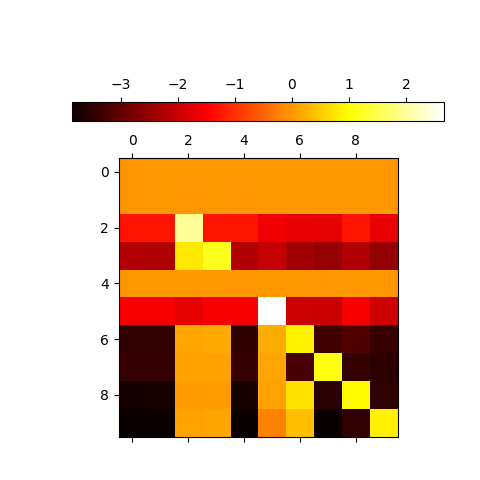}\label{Figure_C2_g}}
		\end{subfigure}
	\end{minipage}
	
	\caption{Empirical verification regarding the learning of the observed reachability matrix.}
	\label{Figure_C2}
\end{figure}

\vspace{-2mm}
\paragraph{Learning the reachability matrix.} Similar to the process of learning the adjacency matrix, since only \emph{observed reachable pairs} $(j,k)$ have $\sum_{i} N_{i,j,k} > 0$, the gradient descent learning paradigm %$\bm{W}^{V}_{(j',k')}$ will keep decreasing when
	%$(j',k')$ is not an observed reachable pairs in the training dataset $\mathcal{D}$, %that is, there is no path in $\mathcal{D}$ in which $j'$ is the target and $k'$ is a non-source node on the path.
	%On the other hand, when $(j,k)$ is indeed an observed reachable pair, $\bm{W}^{V}_{(j,k)}$ may not keep decreasing. This 
 tends to make the $\bm{W}^{V}_{(j,k)}$ terms corresponding to observed reachable pairs $(j,k)$ higher than the $\bm{W}^{V}_{(j',k')}$ terms corresponding to non-observed reachable pairs $(j',k')$ (which is either not reachable or not observed) after the training. 
 In this way, the Transformer model \emph{%learns 
 captures the structural information of observed reachability matrix} 
 %in 
 with weight matrix $\bm{W}^{V}$. 

Figure \ref{Figure_C2} shows the correlation between $\bm{W}^{V}$ and the observed reachabilities under different dataset $\mathcal{D}$'s in the above experiment. Figure \ref{Figure_C2_a} is the real reachability matrix of the graph; Figure \ref{Figure_C2_b}, \ref{Figure_C2_d}, \ref{Figure_C2_f} are the observed reachability matrices in datasets $\mathcal{D}_1$, $\mathcal{D}_2$, $\mathcal{D}_3$, respectively; and Figure \ref{Figure_C2_c}, \ref{Figure_C2_e}, \ref{Figure_C2_g} are the $\bm{W}^{V}$ matrices with training datasets $\mathcal{D}_1$, $\mathcal{D}_2$, $\mathcal{D}_3$, respectively.
%; Figure \ref{Figure_C2_d} is the observed reachability matrix in dataset $\mathcal{D}_2$, and Figure \ref{Figure_C2_e} is the $\bm{W}^{V}$ matrix under $\mathcal{D}_2$; similarly, Figure \ref{Figure_C2_f} is the observed reachability matrix in dataset $\mathcal{D}_3$, and Figure \ref{Figure_C2_g} is the $\bm{W}^{V}$ matrix under $\mathcal{D}_3$. %Note that the model cannot learn reachability $\bm{W}^{V}_{(j,k)}$ for $k=0$ because for any $i,j$, $N_{i,j,0} = 0$ (since there is no edge that ends at $0$). 
%As we can see, all the matrices $\bm{W}^{V}$'s can learn information of the observed reachabilities in the training datasets satisfactorily, 
%	but they cannot deduce any non-observed reachabilities.
 These illustrations 
show that
%As we can see, 
all the weight matrices $\bm{W}^{V}$ can satisfactorily learn the %structural 
information of the observed reachabilities 
present in the training datasets, but cannot deduce any non-observed reachabilities.

\vspace{-2mm}
\paragraph{Predicting the next node on a path.}
From Eq.\eqref{eq:predweight}, the probability vector for predicting the next node is  
	$\mathbf{softmax}(\bm{W}^{M}_{(u_n,:)} + \bm{W}^{V}_{(u_2,:)})$, where $u_n$ represents the current node, and $u_2$ represents the target node.
This %provides an intuitive explanation on why $\bm{W}^{M}$ is learning the observed adjacency matrix while $\bm{W}^{V}$ is learning the observed reachability matrix, since
	%we need to use $\bm{W}^{M}$ on the current node $u_n$ to provide information about the next node $u_n$ connects to, while we need to use
	%$\bm{W}^{V}$ on the target node to provide information on which nodes can reach the target node.
%The softmax operation $\mathbf{softmax}(\bm{W}^{M}_{(u_n,:)} + \bm{W}^{V}_{(u_2,:)})$ 
resembles the procedure in Algorithm~\ref{alg:gt}:
	it predicts the next node $k$ such that both $\bm{W}^{M}_{(u_n,k)}$ is high (corresponding to $\bm{A}_{(u_n,k)} = 1$) and 
		$\bm{W}^{V}_{(u_2,k)}$ is high (corresponding to $\bm{R}_{(u_2,k)} = 1$).

%In summary, Theorem \ref{Thm_1} %our theoretical analysis 
%shows that, a simplified 1-layer 1-head auto-regressive 
%	Transformer (with perfect attention) can learn the important adjacency and reachability information from the training data
%	through gradient descent training, and can predict the next node using these learned information just as what a human algorithm designer would do.
%This suggests that the Transformer, when tackling the path finding or more general the planning task with a given goal, would learn the information relating the next step to both the current step and the
%	goal in order to generate the next task step.
%However, the limitation that the Transformer can only learn the observed reachability but fail to deduce the real reachability suggests that the goal-oriented information the Transformer can learn
%	may be limited and different from what we expect: it may fail to obtain new reachability deduced from the transitivity of the reachability relations.
%In the next section, we will further verify the effectiveness of Transformer on the path-finding task through extensive empirical evaluations.

In summary, our theoretical analysis 
%shows 
demonstrates that a simplified one-layer, one-head autoregressive 
	Transformer (with perfect attention) can effectively learn crucial
 %the important 
 adjacency and reachability information from the training data through gradient descent training.
 Moreover, it can utilize
 this learned information
 to predict the next node %using these learned information just as what a human algorithm designer would do in this case. 
 akin to the decision-making process of a human algorithm designer in similar scenarios.
This suggests that,
 when 
 %tackling
 confronted with the path-finding or more general planning task with a given goal, 
 the Transformer learns the 
 structural information %relating 
 to associate the next step 
 with both the current step and the
	goal,
 %in order to
 enabling it to generate the %next
 subsequent task step.
 Nevertheless, the Transformer's 
 limitation 
 %that the Transformer can only 
 in learning only the observed reachability matrix---without deducing 
 %the real 
 the complete reachability matrix---hints at potential constraints on the goal-oriented information 
 it can acquire.
This limitation may result in the Transformer failing to grasp novel reachability relationships derived from the transitivity of reachability relations, unlike human intelligence.

\vspace{-2mm}
\section{Empirical Evaluations: Peeking into a Trained Transformer}
\label{sec:empirical}

\vspace{-2mm}
In this section, we conduct extensive experiments on the path-finding task using the general Transformer architecture as described in Section \ref{sec:structure}. The datasets are generated as described below. %All the codes are XX.
%We conduct extensive experiments on the path-finding task using the general Transformer architecture as described in Section \ref{sec:structure}.
%The experiments include tests on the overall accuracy of the Transfomer model for the path-finding task, as well as investigation on how well the model learns the attention and the information on
%	the adjacency and reachability matrices.
%In this section, we present these empirical evaluation results, which show that the results derived from our theoretical analysis in Section \ref{sec:theoretical} can be carried over to the general
%	Transformer architecture.

%\subsection{Datasets}

%We first introduce the general datasets used in these experiments. 

%\subsubsection{Graphs}

%\paragraph{Graphs.} 
The DAG is generated randomly based on two parameters: the number of nodes $n$, and the probability of edge $p = 0.1$: %Given these two parameters, %we generate a DAG with $n$ nodes as follows: 
For any $1 \le i < j \le n$, there is an edge $(i,j) \in \mathcal{E}$ with probability $p$. %, %and the randomness for different edges are independent.
%
%
%\subsubsection{Training Data and Test Data}
%
%\paragraph{Training Data and Test Data.} 
Given the DAG, we first find all the possible reachable pairs $(s,t)$. %(i.e., $s\ne t$ and there exists at least one path that starts at $s$ and ends at $t$). 
Then these reachable pairs are separated into 
	the training set (w.p. 0.5) and the test set (w.p. 0.5), but if edge $(s,t)\in \mathcal{E}$, we always put $(s,t)$ in the training set. 
For a reachable pair $(s,t)$ in the training set, we generate $m=20$ random paths that start at $s$ and end at $t$, and put these $m$ paths into the training dataset. 
%
%Moreover, if $(s,t)\in \mathcal{E}$, we add one more path ``$s\ t\ s\ t$'' to make sure that every edge appears at least once in the training data. 
When generating the random path, at each current node $i$, we find all the possible $k\in \mathcal{V}$ such that $\bm{A}^{\text{true}}_{(i,k)} = 1$ and $\bm{R}^{\text{true}}_{(t,k)} = 1$ (i.e., there is an edge $(i,k)\in \mathcal{E}$, and $k$ could also reach the target node $t$), and uniformly choose a random one in them. 
Moreover, if $(s,t) \in \mathcal{E}$, we always put the one-edge path ``$s$ $t$ $s$ $t$ $\backslash$n'' in the training dataset to guarantee that all edges appear at least once in the training data.

%Here are some commonly-used method to generate train paths and test paths:
%\begin{itemize}
%    \item The train paths are all the possible paths, and the test paths are the same. In this case, we are testing the model's ability of memorizing paths.
%    \item The train paths are given by repeatedly choosing a random source node and a random target node, then do a random walk from source node to target node. The testing paths are all the source node and target node pairs that do not appear in the train paths.
    
    %and then do a random walk. 
%    \item ===More===
%\end{itemize}

\subsection{Accuracy on Test Dataset}

We train Transformer models on the aforementioned training dataset 
and subsequently evaluate the performance of
these models 
using the pairs in the test dataset.
The correctness of a
model's output is determined based on its validity in terms of syntax and 
whether it corresponds to a valid path from $s$ to $t$.
In our experiments, we employ Transformer models with an embedding size of $d = 120$.
%For the Transformer models, we set the embedding size $d = 120$, and
We conduct tests using various configurations, ranging from
1-layer and 1-head to 6-layer and 6-head, while considering different graph sizes,
with number of nodes $n$ ranging 
from 100 to 500. 
The accuracy results take average over 10000 trials, and are 
%reported
presented in Figure \ref{fig1:five_images} (due to space limits, some results are deferred to Appendix \ref{sec:complete_exp}, and all of them are consistent with our conclusions).
%We first show a general accuracy result in Figure \ref{fig1:five_images}. In these results, we test Transformer architecture from 1 layer and 1 head to 6 layers and 6 heads, and consider different graph size (number of nodes) from 100 to 500.
From these results, we 
%can observe that 
make the following observations:
%\begin{itemize}
    a) When comparing the figures, the accuracy tends to decrease as the number of nodes increases; %When $n=100, 200$, the accuracy is always above $0.95$; when $n=300$, the accuracy drops to between $0.91$ and
    %	$0.95$; when $n=400, 500$, the accuracy further drops to between $0.80$ and $0.90$ in most cases; 
    b)  When examining each row, the accuracy remains relatively stable even as the number of attention heads increases;
    c)  When examining each column, the accuracy shows at most a slight improvement as the number of layers increases. 
    
    %when the embedding size is large enough comparing to the graph size (e.g., $n=100,200$), the accuracy does not change much when we increase the number of layers; however, when the embedding size is small comparing to the graph size (e.g., $n=300,400$), the accuracy increases a little as we increase the number of layers, i.e., the accuracy with $L=3,4,5,6$ layers is a little higher than the accuracy with $L=1,2$ layers.
%\end{itemize}

\begin{figure}[t]
    \centering
    % Top row with three images
    \begin{subfigure}[100 Nodes]{
        \centering
        \includegraphics[width=0.208\linewidth]{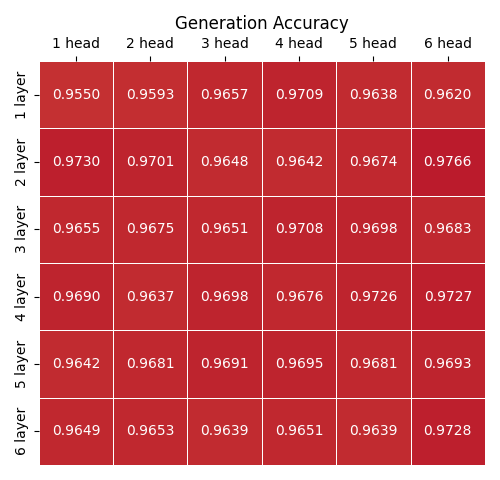}
        \label{fig1:image1}}
    \end{subfigure}
    \begin{subfigure}[200 Nodes]{
        \centering
        \includegraphics[width=0.208\linewidth]{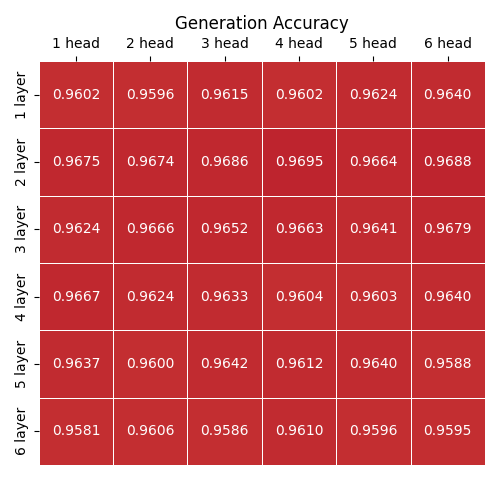}
        \label{fig1:image2}}
    \end{subfigure}
    \begin{subfigure}[300 Nodes]{
        \centering
        \includegraphics[width=0.208\linewidth]{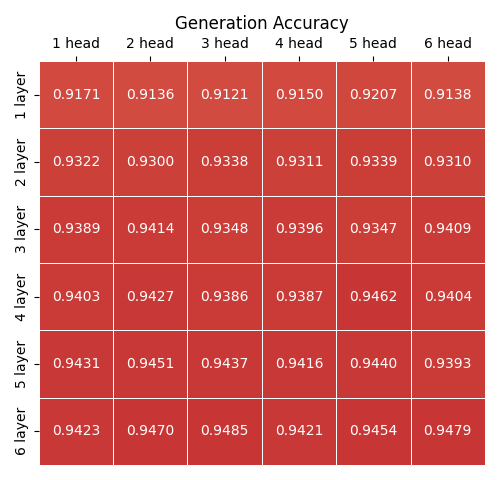}
        \label{fig1:image3}}
    \end{subfigure}
    %\vspace{1cm} % Space between the two rows
    % Bottom row with two images
    \begin{subfigure}[400 Nodes]{
        \centering
        \includegraphics[width=0.266\linewidth]{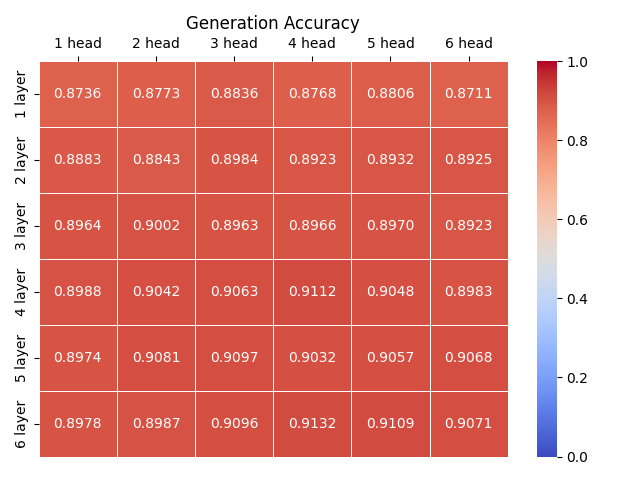}
        \label{fig1:image4}}
    \end{subfigure}
    %\begin{subfigure}[500 Nodes]{
    %    \centering
    %    \includegraphics[width=0.30\linewidth]{Fig/E2/Embd=120_Temp=1_Nodes=500.png}
    %    \label{fig1:image5}}
    %\end{subfigure}

    \caption{Accuracy on the test datasets with embedding size $d= 120$.} %\textbf{Takeaway}: when the embedding size is large enough, even a Transformer with 1 layer and 1 head could achieve good accuracy; when the embedding size is not large enough, even a Transformer with 6 layers and 6 heads is not enough to achieve good accuracy. }
    \label{fig1:five_images}
    \vspace{-0.35cm}
\end{figure}

\begin{figure}[t]
	\centering
 \vspace{-0.30cm}
	% Top row with three images
	\begin{subfigure}[100 Nodes]{
			\centering
			\includegraphics[width=0.2\linewidth]{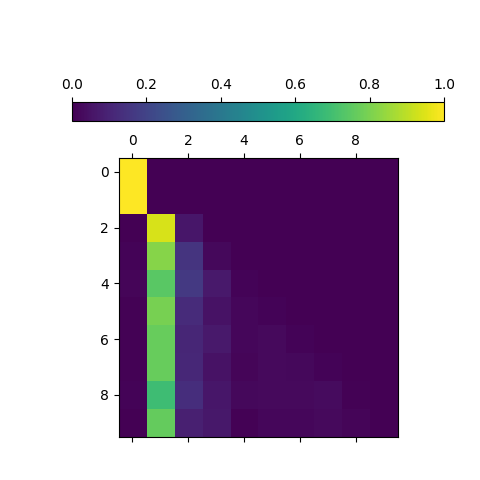}
			\label{fig2:image1}}
	\end{subfigure}
	\begin{subfigure}[200 Nodes]{
			\centering
			\includegraphics[width=0.2\linewidth]{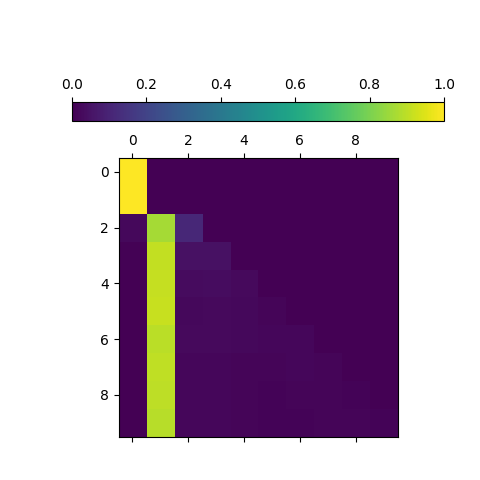}
			\label{fig2:image2}}
	\end{subfigure}
	\begin{subfigure}[300 Nodes]{
			\centering
			\includegraphics[width=0.2\linewidth]{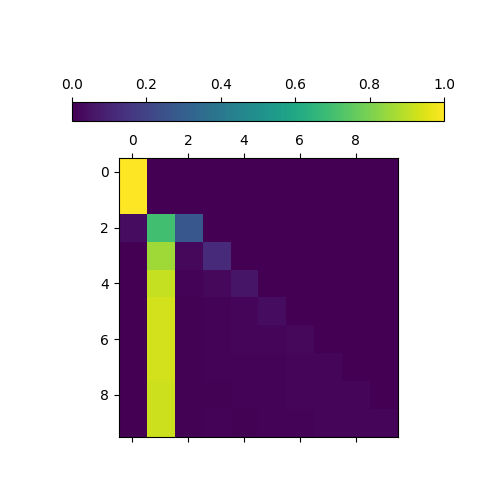}
			\label{fig2:image3}}
	\end{subfigure}
	%\vspace{1cm} % Space between the two rows
	% Bottom row with two images
	\begin{subfigure}[400 Nodes]{
			\centering
			\includegraphics[width=0.2\linewidth]{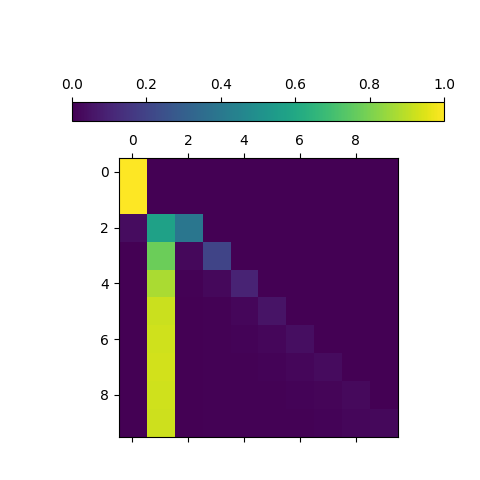}
			\label{fig2:image4}}
	\end{subfigure}
	%\begin{subfigure}[500 Nodes]{
	%		\centering
	%		\includegraphics[width=0.30\linewidth]{Fig/E3/Att_500_1_1_120_Layer1_Head1.png}
	%		\label{fig2:image5}}
	%\end{subfigure}
	
	\caption{The average attention in the 1-layer and 1-head Transformers.} %\textbf{Takeaway}: the Transformer could let its attention concentrate on the target node, especially when the embedding size is large enough.}
	\label{fig2:five_images}
 \vspace{-0.35cm}
\end{figure}

The above observations suggest that the embedding size is the most important hyperparameter that affects the accuracy of the model. % when the embedding size is large enough comparing to the graph size, even 1-layer 1-head models would perform well; 
%
%and the number of layers affect the accuracy of the model, since they all determine the number of parameters of the model.
On the one hand, when the embedding size is sufficiently large compared to the graph size, even 1-layer and 1-head models perform well.
This coincides with our theoretical analysis, which shows that when the embedding size equals to the graph size, 
	the 1-layer and 1-head structure is enough to predict the next nodes accurately. 
%Therefore, increasing the number of layers does not improve the performance significantly. 
%However, our empirical results show that 
%
On the other hand, when the embedding size is small compared to the graph size, even 6-layer and 6-head Transformers cannot achieve good performance. Because of this, in the following, we concentrate on the explainability of the 1-layer and 1-head Transformer models.
%
%the Transformer architecture may need more parameters to increase the accuracy, which could be achieved by
%	increasing the number of layers in the model. 
%However, increasing the number of layers does not efficiently increase the accuracy a lot, e.g., the accuracy for Transformer with 3 layers is almost the same as the accuracy for Transformer with 6 layers.
%\siwei{maybe some explanation here?}

%For example, a 6-layer transformer has more than 4 times parameter i

\subsection{Peeking into a Trained Transformer}

%We now look inside the Transformer models, and try to find more evidence that our theoretical analysis reflects the reality.

%
\paragraph{Attention.} 
In our analysis, we assume that the attention is fixed on the target node.
%we first assume that the attention is fixed at the target node. 
Is this true for the Transformer models learned from real data?
The corresponding results are shown in Figure \ref{fig2:five_images}. 
These results are obtained by looking into the attention mechanism of the 1-layer and 1-head Transformer models, 
	and showing the average (taking on the test dataset) matrix of $\textbf{softmax}\left(\frac{{\bm{Q}\bm{K}^\top}}{\sqrt{d_k}}  \right)$, 
	of which the $n$-th row represents the attention vector for predicting the $(n+1)$-th token.

Note that the second column in these figures represents the attention weights on the second token, which corresponds to the target node in our test data.
We can see that, when predicting the next tokens, almost all the attention weights are concentrated on this column, especially for those models
	with higher accuracy (Figure \ref{fig2:image1} for $n=100$ and Figure \ref{fig2:image2} for $n=200$). 
This demonstrates that indeed the Transformer model learns the correct attention for the path-finding task, and our assumption on the attention weights for the theoretical analysis is reasonable.

\begin{figure}[t]
	\centering
	% Top row with three images
	\begin{subfigure}[100 Nodes]{
			\centering
			\includegraphics[width=0.21\linewidth]{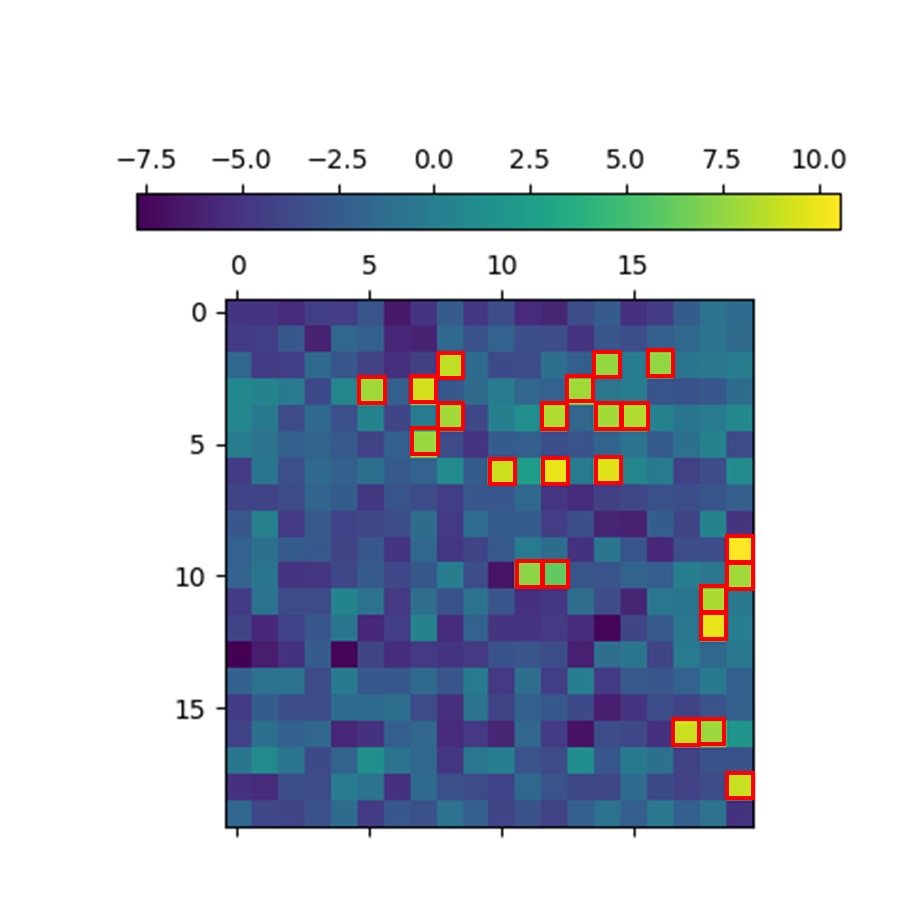}
			\label{fig3:image1}}
	\end{subfigure}
	\begin{subfigure}[200 Nodes]{
			\centering
			\includegraphics[width=0.21\linewidth]{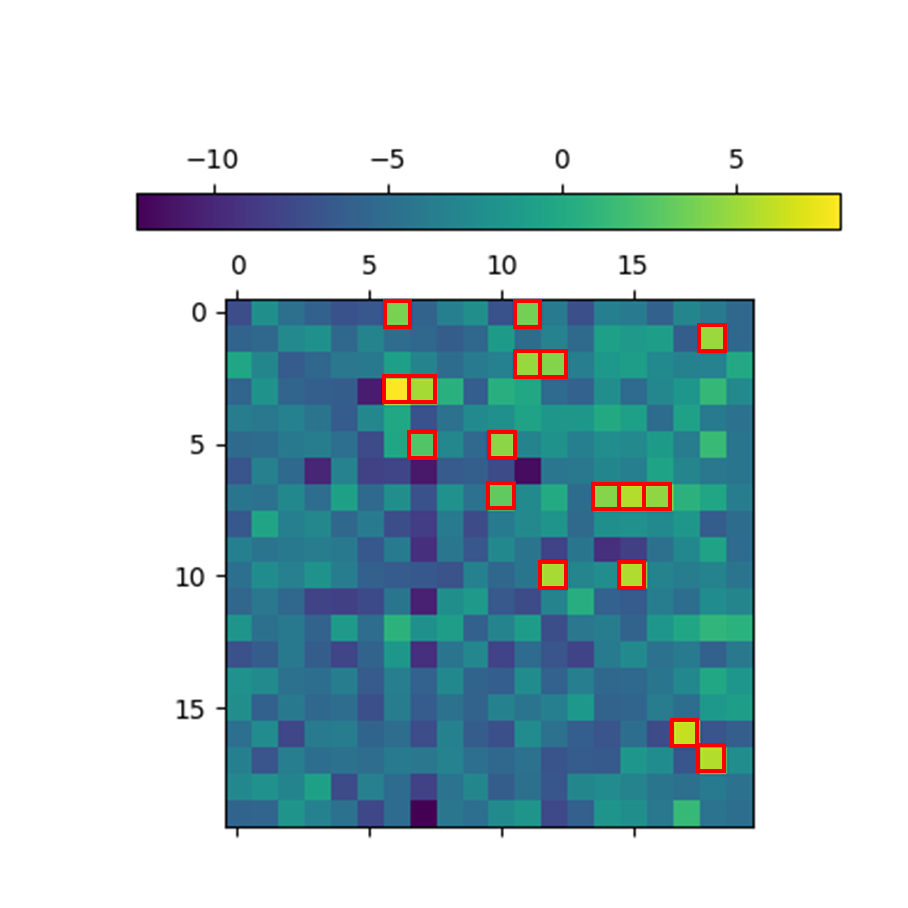}
			\label{fig3:image2}}
	\end{subfigure}
	\begin{subfigure}[300 Nodes]{
			\centering
			\includegraphics[width=0.21\linewidth]{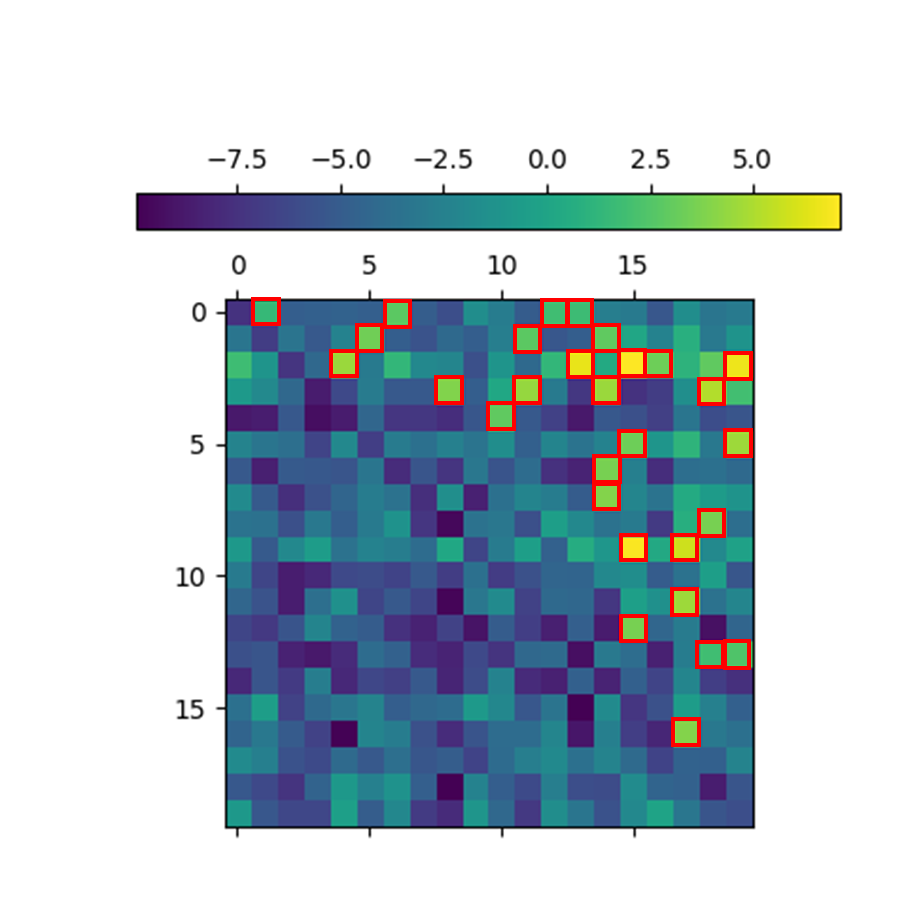}
			\label{fig3:image3}}
	\end{subfigure}
	%
	%\vspace{1cm} % Space between the two rows
	%
	% Bottom row with two images
	%\begin{subfigure}[400 Nodes]{
	%		\centering
	%		\includegraphics[width=0.28\linewidth]{Fig/E4/New_MLP_weight_1_1_120_400_20.png}
	%		\label{fig3:image4}}
	%\end{subfigure}
	%\begin{subfigure}[500 Nodes]{
	%		\centering
	%		\includegraphics[width=0.28\linewidth]{Fig/E4/New_MLP_weight_1_1_120_500_20.png}
	%		\label{fig3:image5}}
	%\end{subfigure}
        \begin{subfigure}[Average Weight Gap]{
			\centering
			\includegraphics[width=0.26\linewidth]{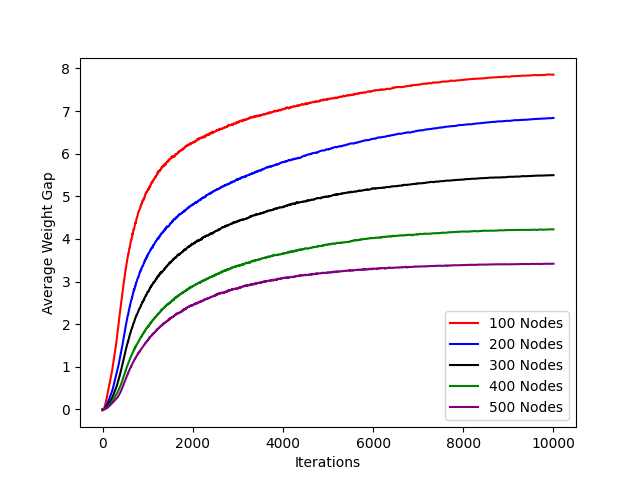}
			\label{fig3:image6}}
	\end{subfigure}
	
	\caption{The first 20 rows and columns of $\bm{W}^{M'}$ %in 1-layer and 1-head Transformer 
 (the red boxes correspond to $1$'s in the adjacency matrix), and the average weight gap between edge terms and non-edge terms in $\bm{W}^{M'}$.}
			%$\bm{A}$.\siwei{this description seems a little strange}  
   %Takeaway: the adjacency matrix is stored in the $\text{FFN}$ layer.}
	\label{fig3:five_images}
\end{figure}

%\subsection{Adjacency Matrix}

%Our analysis %in Section~\ref{sec:theoretical} 
%shows that the observed adjacency matrix and the observed reachability matrix are stored in the feed-forward layer and the attention layer. We now verify this on the general 1-layer and 1-head Transformer model that includes all the Transformer mechanisms, such as attention weights, non-linear transformation, 
%	token and position embedding.  

\vspace{-2mm}

\paragraph{Adjacency Matrix.}
%In this subsection, we want to verify that the adjacency matrix is stored in the feed-forward layer.
In the 1-layer and 1-head Transformers, let $\bm{W}^{M'}$ (shown in Figure \ref{fig3:five_images}) be the matrix whose $i$-th row is $\text{FFN}\left({\bm e}_i^\top \bm{W}_t\right)\bm{W}_o + ({\bm e}_i^\top \bm{W}_t)\bm{W}_o$, where ${\bm e}_i$ is the one-hot column vector that represents the token for node $i$. %(with dimension $M = |\mathcal{V}| + 1$). 
Based on the Transformer computation, intuitively this matrix is one of the components in the output that contains information related to the current node. The detailed reason for choosing this matrix is explained in Appendix \ref{sec:rea_matrix}.

In Figure \ref{fig3:image1}, the $\bm{W}^{M'}$ matrix and the adjacency matrix are highly aligned: all large entries in the $\bm{W}^{M'}$ matrix correspond to real edges, and all real edges correspond to large entries in the $\bm{W}^{M'}$ matrix. 
This high accuracy is because the embedding size $d=120$ is higher than the number of nodes $n=100$. If the embedding size is lower than the graph size (Figures \ref{fig3:image2}, \ref{fig3:image3}), we inevitably lose some accuracy when approximating the adjacency matrix by the product of matrices with rank smaller than the graph size. %let alone the non-linear layers' influence. 
Even so, there is still high relevance between $\bm{W}^{M'}$ and the adjacency matrix: almost all real edges correspond to large entries in the $\bm{W}^{M'}$ matrix. 

In Figure \ref{fig3:image6}, we show the gap between the average weight corresponding to edges (i.e., the average of $\bm{W}^{M'}_{(i,j)}$'s with $i < j$ and $(i,j) \in \mathcal{E}$) and the average weight corresponding to non-edges (i.e., the average of $\bm{W}^{M'}_{(i,j)}$'s with $i < j$ and $(i,j) \notin \mathcal{E}$) during the training process. %We can see that in all these five graphs, 
These gaps keep increasing until convergence, suggesting that
weights between edges and non-edges are more easily separated as the learning process proceeds. %Moreover, with lower number of nodes, the gap is higher. We believe this is because that when the embedding size is fixed, one can approximate the adjacency matrix better for smaller graphs.

%\wei{Currently, the comparison is only on the visual inspection of the figures, and they are only for the first 20 nodes. Can we provide an objective measure, such as the cosine similarity, between the
%	real adjacency matrix and the empirical one? This would give an overall picture perhaps, and could also show the trend that the similarity is going down with the increase of the graph size?
%	To put the cosine similarity into perspective, we can also give a baseline similarity measure between a random adjacency matrix with the same parameter and the empirically learned one, to see how
%	learning in the Transformer elevates the similarity.
%	One way to illustrate these results is to show a figure with x-axis as the number of nodes, and y-axis as similarity, and there would be two curves, one for similarity with
%	the real adjacency matrix, another for similarity with a random adjacency matrix.
%}
%

%\subsection{Reachability Matrix}
\vspace{-2mm}

\paragraph{Reachability Matrix.} 
%In this subsection, we want to verify that the reachability matrix is mainly stored in the attention layer. %i.e., in the two terms $\bm{X}_2 \bm{W}^V\bm{W}_o + \text{FFN}\left(\bm{X}_2 \bm{W}^V \right)\bm{W}_o$. 
%Suppose that $\bm{X}_2$ is the $i$-th row in $W_t$.
In the 1-layer and 1-head Transformers, let $\bm{W}^{V'}$ be the matrix whose $i$-th row is $({\bm e}_i^\top\bm{W}_t) \bm{W}^V\bm{W}_o + \text{FFN}\left(({\bm e}_i^\top\bm{W}_t) \bm{W}^V \right)\bm{W}_o$, where ${\bm e}_i$ is the one-hot column vector that represents the token for node $i$.
Intuitively, this matrix is the remaining component in the output that contains information related to the target node. The detailed reason is also explained in Appendix \ref{sec:rea_matrix}.
%(with dimension $M = |\mathcal{V}| + 1$). 
%Intuitively, this matrix contains information related to the target node, and the detailed reason is given in Appendix \ref{sec:rea_matrix}.
%\wei{same here. Need a simple explanation.}
%In the simplified Transformer model of Theorem \ref{Thm_1}, $\bm{W}^{V'}$ is the same as matrix $\bm{W}^V$. 

%Except for the adjacency matrix, another important point in our analysis is about the reachability: we show that the Transformer could not deduce new reachabilities. 
%\wei{Why do we focus on non-observed reachable pairs? Can we also have some results similar to the adjacency matrix, that is, compare the real observed reachability matrix
%	with the empirical one? Perhaps from $\bm{M}^V$?}
%To demonstrate this, we also try to see whether the attention value matrices look similar to the reachability matrices. 
%In the following, we consider the matrix $\bm{W}^{V'}$ as follows: $\bm{W}^{V'}_{(i,:)}$ is the output weight vector of sending the token represents node $i$ into the token embedding matrix ($\bm{W}_t$), the value matrix in the MHA layer ($\bm{W}^{V}$), and the output weight matrix ($\bm{W}_o$). Note that in the setting of our analysis, this is the same as matrix $\bm{W}^V$. 

In Figure \ref{fig4:image6}, we show the average weights of three different sets in the graphs: ``obs'' corresponds to the $\bm{W}^{V'}_{(t,k)}$'s with $t \ge k$ and $\bm{R}^{\text{obs}}_{(t,k)} = 1$; ``real$\setminus$obs'' corresponds to the $\bm{W}^{V'}_{(t,k)}$'s with $t \ge k$, $\bm{R}^{\text{obs}}_{(t,k)} = 0$ but $\bm{R}^{\text{real}}_{(t,k)} = 1$; and ``non'' corresponds to the $\bm{W}^{V'}_{(t,k)}$'s with $t \ge k$ and $\bm{R}^{\text{real}}_{(t,k)} = 0$. 
Here we only show the results of graphs with 100 nodes and 200 nodes, since their accuracy is high enough %(about 0.96, and does not become higher even if there are more layers/heads, as shown in Figures \ref{fig1:image1} and \ref{fig1:image2}); b) 
and their attention is quite close to being concentrated on the target node. %(see Figures \ref{fig2:image1} and \ref{fig2:image2}). 
When there are more nodes, the ability to approximate the reachability matrix is not enough for us to distinguish it.  
From these average weights, we can see that the Transformer learns $\bm{R}^{\text{obs}}$ quite well, as for those terms in ``real$\setminus$obs'', their weights are almost the same as those in ``non''. 
This echoes our analysis. %in Section~\ref{sec:theoretical}. 

To further demonstrate that $\bm{R}^{\text{real}}$ is not learned as good as $\bm{R}^{\text{obs}}$, we divide the source-target node pairs $(s,t)$ in the test dataset into four categories:
%\begin{itemize}
    a) degree 0: $\bm{R}^{\text{obs}}_{(t,s)} = 1$;
    b) degree 1: $(s,t)$ is not of degree 0, while $s$ has at least one out-neighbor node $u$ such that $(u,t)$ is of degree $0$, i.e. $\bm{R}^{\text{obs}}_{(t,u)} = 1$;
    c) degree 2: $(s,t)$ is not of degree 0 and 1, while $s$ has at least one out-neighbor node $u$ such that $(u,t)$ is of degree 1;
    d) degree 3 or more: the remaining $(s,t)$ pairs in the test dataset.
%\end{itemize}
%
%Roughly speaking, %in our analysis, when the Transformer is predicting the next node of $s$ for the source-target pair $(s,t)$, 
%	it will add the adjacency vector of $s$ and the reachability vector of $t$, and use the sum as the weight vector of the next node. 
Roughly speaking, in our analysis, for $(s,t)$ pairs of degree 0 or 1, we know that there is a node $u$ such that $\bm{A}^{\text{obs}}_{(s,u)} = 1$ and $\bm{R}^{\text{obs}}_{(t,u)} = 1$. 
Then node $u$ will have a large weight, indicating a high accuracy. 
As for $(s,t)$ pairs of degree 2 or more, there is no node $u$ such that both $\bm{A}^{\text{obs}}_{(s,u)} = 1$ and $\bm{R}^{\text{obs}}_{(t,u)} = 1$. 
%
%its corresponding entry in the adjacency vector of $s$ and 
%	its corresponding entry in the reachability vector of $t$ are large. 
In this case, the high-weight entry when predicting the next node of $s$ is either an adjacent node of $s$ or a recorded node that can reach $t$. 
This should reduce the accuracy. 

%
%We demonstrate the existence of this phenomenon in Figure \ref{fig4:five_images} (b)-(d). 
%Since none of the recorded nodes reachable to $t$ is adjacent to $s$, the accuracy for $(s,t)$ pairs of degree 2 or more should be much lower than those of degree 0 or 1. 
%
To see this, we check the accuracy of the Transformers on the $(s,t)$ pairs of the four different categories. The results are shown in Figure \ref{fig4:five_images} (b)-(d). 
In these figures, each row of the accuracy matrix is further divided into four sub-rows corresponding to the accuracy of degree-0 pairs, degree-1 pairs, degree-2 pairs, and degree-3 or more pairs
	respectively (in the graph with 100 nodes, there are no test $(s,t)$ pairs in the degree-3 or more category).
From these results, we can see that the accuracy for degree-2 pairs and degree-3 or more pairs is much lower than the two other categories in most cases. 
It indicates that, even with more parameters and a more complex structure (e.g. a 6-layer and 6-head Transformer), the Transformer model has a fundamental difficulty in generating paths for high-degree source-target pairs,
	namely those pairs %with existing paths but 
 that can only be connected by concatenating several path segments in the training dataset.
%In short, the model fails to learn reachability through the transitivity among the reachability relations. 
This result demonstrates the validity of our theoretical analysis, %which shows that 
i.e., after training with gradient descent on cross-entropy loss, the Transformer can only learn observed reachability, and will miss those unobserved reachability
	deduced from the transitivity of the reachability relation.

%We also try to see whether the attention matrices look similar to the reachability matrices. 
%In the following, we consider the matrix $\bm{W}^{V'}$ as follows: $\bm{W}^{V'}_{(i,:)}$ is the output weight vector of sending the token represents node $i$ into the token embedding matrix ($\bm{W}_t$), the value matrix in the MHA layer ($\bm{W}^{V}$), and the output weight matrix ($\bm{W}_o$). Note that in the setting of our analysis, this is the same as matrix $\bm{W}^V$. 

%In Figure \ref{fig4:image6}, we show the average weights of three different sets in the five graphs: obs corresponds to the $\bm{W}^{V'}_{(t,k)}$'s with $t \ge k$ and $\bm{R}^{\text{obs}}_{(t,k)} = 1$; real$\setminus$obs corresponds to the $\bm{W}^{V'}_{(t,k)}$'s with $t \ge k$, $\bm{R}^{\text{obs}}_{(t,k)} = 0$ but $\bm{R}^{\text{real}}_{(t,k)} = 1$; and non corresponds to the $\bm{W}^{V'}_{(t,k)}$'s with $t \ge k$ and $\bm{R}^{\text{real}}_{(t,k)} = 0$. From these average gaps, we can see that the Transformer learns $\bm{R}^{\text{obs}}$ quite well. $\bm{R}^{\text{real}}$ is also learned, but not as good as $\bm{R}^{\text{obs}}$. Because of this, the accuracy for degree-2 pairs is much lower. 
%sending a
%Thus, in our tests we need to print $\bm{W}_t\bm{W}_1\bm{W}_2\bm{W}_o$ and compare it with the adjacency matrix.

%\subsection{Summary of the Empirical Results}

In summary, our extensive empirical evaluation leads to the following conclusions about the Transformer model in achieving the path-finding task:
(a) With large enough embedding size, the model can achieve high accuracy in general; %, with large enough embedding size;% and sufficient number of layers;
(b) The model achieves its performance by concentrating attention on the target nodes as intended, and learning the information on adjacency and reachability matrices, just as
	what a human would do and as predicted by our theoretical analysis; and
(c) The model may have limitations and fail to learn high-order reachability relations through transitivity, and thus fail to generate paths derived from high-order reachability.
%	as human would expect.

%This accords with our analysis: for a source-target pair $(s,t)$ with degree 2, it does not know where to go from $s$ (since the reachabilities of all the adjacent nodes of $s$ to $t$ are not in the train dataset). However, For degree 1, one of its adjacent nodes' reachability to $t$ is given in the train dataset, hence the accuracy is much higher. 

%Note that in our datasets, there are few degree 2 source-target pairs (hundreds of degree 0 pairs or degree 1 pairs, but less than 10 degree 2 pairs), hence the accuracy for degree 2 pairs are not very stable. 

\begin{figure}[t]
    \centering
    % Top row with three images
    \begin{subfigure}[Average Weights]{% in $\bm{W}^{V'}$]{
        \centering
        \includegraphics[width=0.24\linewidth]{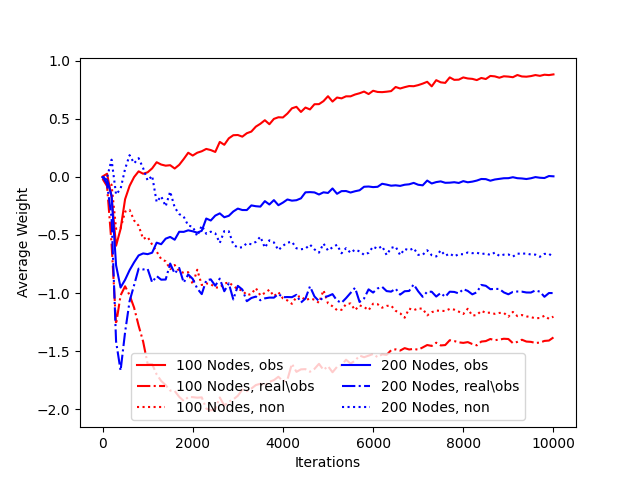}
        \label{fig4:image6}}
    \end{subfigure}\begin{subfigure}[100 Nodes]{
        \centering
        \includegraphics[width=0.20\linewidth]{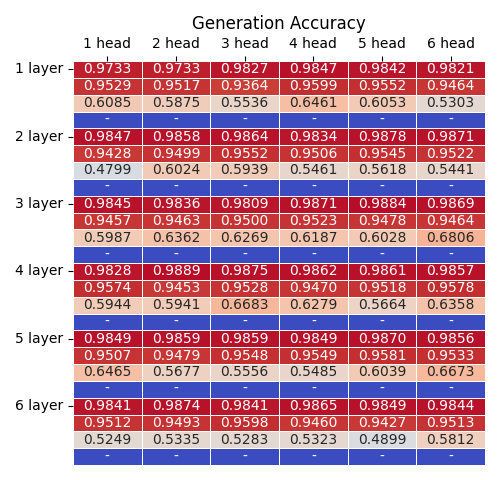}
        \label{fig4:image1}}
    \end{subfigure}
    \begin{subfigure}[200 Nodes]{
        \centering
        \includegraphics[width=0.20\linewidth]{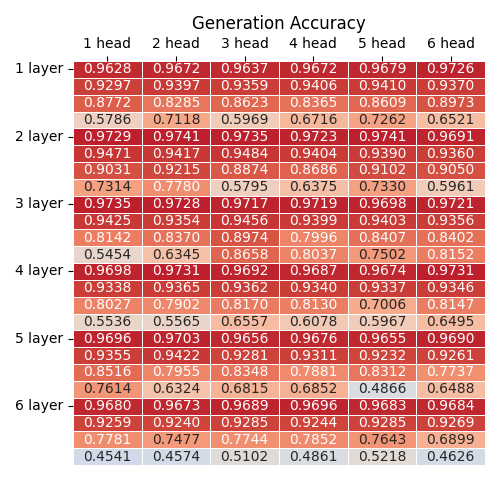}
        \label{fig4:image2}}
    \end{subfigure}
    \begin{subfigure}[300 Nodes]{
        \centering
        \includegraphics[width=0.256\linewidth]{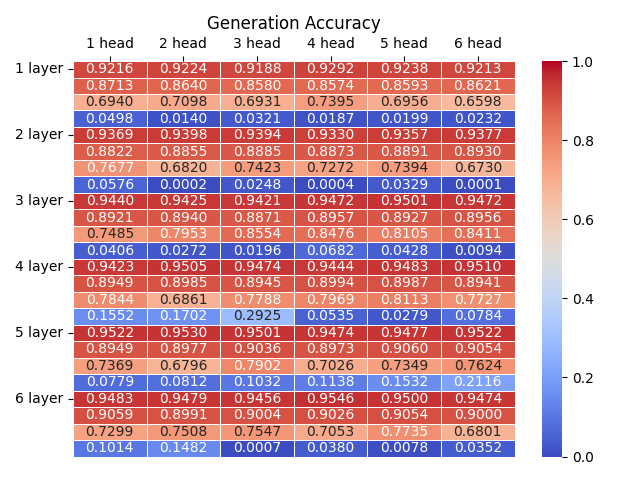}
        \label{fig4:image3}}
    \end{subfigure}
    %\vspace{1cm} % Space between the two rows
    % Bottom row with two images
    %\begin{subfigure}[400 Nodes]{
    %    \centering
    %    \includegraphics[width=0.30\linewidth]{Fig/E5_back_up/DF_Embd=120_Temp=1_Nodes=400.png}
    %    \label{fig4:image4}}
    %\end{subfigure}
    %\begin{subfigure}[500 Nodes]{
    %    \centering
     %   \includegraphics[width=0.30\linewidth]{Fig/E5_back_up/DF_Embd=120_Temp=1_Nodes=500.png}
     %   \label{fig4:image5}}
    %\end{subfigure}

    \caption{The average weights in $\bm{W}^{V'}$, and the accuracy for $(s,t)$'s with different degrees.} %\textbf{Takeaway}: the Transformer can only store the observed reachabilities %(in both the $\text{MHA}$ and the $\text{FFN}$ layers), 
    %and cannot learn the non-observed reachabilities well.}
    \label{fig4:five_images}
\end{figure}

\section{Discussion and Future Work}
\label{sec:discussion}

In summary, this paper and Project ALPINE more broadly conceptualize planning as path-finding in networks, and combine theoretical analysis of the Transformer architecture and autoregressive loss with empirical validation. Our aim is to uncover the mechanisms by which intelligence may emerge from autoregressive Transformer architectures. We analytically demonstrate that Transformers possess the expressiveness required to perform path-finding tasks and that gradient descent on cross-entropy loss enables them to learn necessary—but incomplete—graph information for path-finding. Additionally, we reveal both analytically and empirically that autoregressive training of language models has inherent limitations in the path-finding task.

\noindent{\bf Practical Implications}:
Our findings in LLMs for path planning may have practical implications for the training, testing, and enhancement of language models. In particular, the limitations we identified in current Transformer architectures for transitive reasoning suggest several directions for enhancing LLM frameworks to achieve more advanced and general planning-and-reasoning capabilities across diverse applications. For instance, in data generation for training, creating more diversified datasets that explicitly cover more %cr%transitive and other 
reachability relationships may help the model %cr%learn these complex structures. 
achieve a higher accuracy. When evaluating a language model's planning capability, it may be beneficial to test for higher-order relationships not directly encoded in the training data but requiring chaining and concatenation to assess whether the model can perform transitive planning. Furthermore, by highlighting limitations in current language models, our study motivates future research into improved Transformer architectures, including incorporating transitivity directly into the model structure.

\noindent{\bf Challenges in Reasoning about Unobserved Reachability}:
%Transitivity is a basic property in logic. Since LLMs are often considered close to AGI, it is natural to question whether they can deduce transitivity. However, we demonstrate that this is not the case for Transformers trained via next-token prediction, since the novel reachability relationships derived from the transitivity of reachability relations are also unobserved reachability and they cannot be recorded according to Theorem \ref{Thm_1}.
Technically, the challenge in learning unobserved reachability with current Transformer architectures stems from the nature of next-token prediction loss: learning unobserved reachability incurs a higher training loss. Specifically, when predicting the next token for a given current node $i$ and target node $j$, the optimal distribution for minimizing training loss should align with the observed distribution in the training dataset,
i.e., $\Pr[\text{next node} = k | \text{current node} = i \text{ and target node} = j] = \frac{N_{i,j,k}}{N_{i,j}}$ (as explained in Section~\ref{sec:dynamics}). 
Learning unobserved reachabilities requires deviating from the distribution defined by the training data, which leads to a higher training loss. Consequently, with the current training loss and Transformer architecture, the model cannot learn unobserved reachabilities, such as transitive reachability. To enable the model to learn transitivity, we may need alternative training objectives, such as path accuracy, or structural improvements to the Transformer that allow it to `deduce' unobserved reachabilities. Conceptually, the current training data and loss objective do not provide sufficient information to teach the model transitivity or other derived relationships. Therefore, enhancing transitivity and similar capabilities may require enriching the training data, modifying the objective function, or incorporating new components into the model architecture.

%\subsection{Future Work}
%\vspace{-2mm}
\noindent{\bf Future Directions}: 
Our investigation opens several promising directions for future research: (a) Extending our study to hyper-graphs and hyper-paths, where a hyper-edge represents scenarios requiring multiple preconditions to be met simultaneously in order to carry out the next step, as often seen in task planning and mathematical proofs.
(b) Addressing the limitations of Transformers in path-finding and other planning tasks by exploring richer path-finding languages, fine-tuning, or architectural improvements to LLMs.
%
%Further understand the limitation of the Transformer on path finding: the current inability of transitive reachability may be due to the limited expressive power of the path language, and it may also
%	suggest the fundamental difference between Transformer's interpretation and human's interpretation on path finding.
%	This requires further study, perhaps by enriching the languages, or by an investigation on the deployed LLMs such as GPT4.
%	It can also stimulates new improvement to the Transformer architecture to achieve better performance.
 (c) Examining connections between the abstract path-finding task and concrete planning tasks (e.g., block manipulation in Blocksworld) to understand whether, and how, Transformers abstract these tasks into path-finding frameworks. 
%	, or find if what Transformer does for concrete planning tasks has any commonalities with the abstract path-planning task.
%(d) Study deployed LLMs or fine-tune LLMs to connect their planning capabilities with our path-finding investigation.
(d) Investigating in-context path-finding capabilities, where training data includes different graphs with corresponding paths, to see how Transformers learn to find new paths in new graphs.
(e) Exploring the integration of chain-of-thought and backtracking capabilities into Transformers for path-finding, which may offer crucial  insights into enabling these features for general search and planning tasks.

In our ongoing project ALPINE, we plan to deepen our investigation into all the aforementioned fronts. 
We also hope that our work will inspire more researchers to study LLMs through combined theoretical and empirical analysis, with the ultimate goal of enhancing their capabilities and understanding how human-like intelligence can be achieved through statistical learning and AI mechanisms.

\clearpage

\bibliographystyle{abbrv}
\bibliography{sample}

\newpage
\appendix

\section{Related Works}
\label{sec:relatedwork}
\subsection{LLMs for Planning}
Several recent studies have empirically evaluated the planning abilities of large language models. For instance, CogEval has introduced a variety of planning tasks set in mazes and graphs, ultimately finding no evidence to suggest LLMs grasp the intricacies of the maps or the planning tasks themselves \cite{momennejad2024evaluating}. Similarly, another study explored the Blocksworld game, a planning challenge where humans typically achieve success rates above $70\%$, in stark contrast to GPT-3's mere $5\%$ success rate \cite{valmeekam2023planning}. Our paper also proposes a novel approach by formulating a class of planning problems as path-finding on graphs, applying this model to the Blocksworld game and uncovering significant insights, as detailed in Appendix~\ref{sec:blocksworld}.

Despite these seemingly negative evaluations, LLMs have shown remarkable aptitude in executing real-world planning tasks, creating the field of \emph{autonomous agents} \cite{wang2024survey}. 
Certain applications of autonomous agents feature explicit graphs. 
In the tool agent HuggingGPT~\cite{shen2024hugginggpt}, LLMs are deployed to trigger a sequence of external APIs in response to user requests. Here, APIs are conceptualized as graph nodes, with their interrelations represented as edges, and the selection process akin to identifying a path or subgraph that aligns with user demands. 
This scenario is an extension of the settings discussed in this paper, where the graph is text-attributed and the objective function is evaluated through textual analysis. 
In addition, the application of graph search techniques has been shown to enhance the performance of tool agents significantly \cite{liu2023controlllm,liu2024toolnet}. 
This demonstrates that our approach of abstracting planning as path-finding in graphs is reasonable.
The math agent AlphaGeometry utilizes LLMs to solve geometry problems \cite{trinh2024solving}.
By treating lemmas as nodes and their interdependencies as edges, the process of finding a proof of a theorem is analogous to finding a path to the theorem node in the above graph formed by possible lemma nodes and their interdependency edges.
However, \cite{trinh2024solving} only focuses on using LLMs to generate auxiliary constructions, and the reasoning tasks are done by a non-LLM engine. This is very different from our approach. 
There are no explicit graphs in other agents, such as game agents \cite{wang2023voyager}, embodied agents \cite{huang2022language}, and code agents \cite{shinn2024reflexion}. The core strategy in these domains is to employ verbal reinforcement learning within LLMs. 
However, it is noteworthy that any dynamic programming problem characterized by deterministic state transitions can be reformulated as a shortest path problem on a graph, with states and transitions represented as nodes and edges, respectively. As a result, the area of autonomous agents is also closely related to the path-finding task investigated in this paper.

\subsection{LLMs for Graphs}
GPT4Graph \cite{guo2023gpt4graph} and NLGraph \cite{wang2024can} have developed extensive frameworks for assessing LLMs in the context of graph tasks. These frameworks encompass a broad spectrum of challenges, including classic graph problems (e.g., connectivity, cycle detection, and topological sorting), graph neural network (GNN) tasks (e.g., node and graph classification), and semantic graph question answering (e.g., knowledge graph inquiries). They also explore various input formats, such as adjacency lists, edge lists, GML, and GraphML, alongside innovative prompting techniques such as few-shot, role prompting, chain-of-thought, and algorithmic prompting (e.g., stating ``we are using DFS''). 
These studies demonstrate that LLMs possess basic graph processing capabilities, and the choice of prompts and formats significantly influences the performance. 
Yet, they also reveal the models' susceptibility to spurious correlations within graphs.
GPT-4, for instance, only achieves around $50\%$ accuracy on shortest path tasks, even when utilizing complex prompts. 
To our knowledge, our paper presents the first theoretical analysis that identifies and explains the spurious correlations learned by transformers, partially supporting some of the negative outcomes reported in these studies.

There has also been a surge in efforts aiming at bolstering LLMs' performance on graph tasks. Innovations such as GraphGPT \cite{tang2023graphgpt} and GraphLLM \cite{chai2023graphllm}, which incorporate an additional GNN encoder, have shown notable improvements across the aforementioned graph tasks. GraphInstruct \cite{luo2024graphinstruct} seeks to enhance LLMs' capabilities using pure LLM approaches. This involves meticulously documenting the steps of classical algorithms (e.g., BFS and DFS) and fine-tuning LLMs to learn these graph algorithms. This method of procedural supervision has extended the capacity of LLMs in graph tasks from the complexity class $\text{TC}^0$ to P/poly \cite{feng2024towards}.
However, while this approach has yielded performance improvements in simpler tasks such as topological sorting and connectivity, it has proven less effective for more complex challenges, e.g., finding Hamiltonian Paths. %This method of procedural supervision has extended the capacity of LLMs in graph tasks from the complexity class $\text{TC}^0$ to P/poly \cite{feng2024towards}.

%GPT4Graph \cite{guo2023gpt4graph}: format and input design, e.g., CoT and format explanation. 

%NLGraph \cite{wang2024can}: format and input design, e.g., CoT and format explanation. 

%GraphInstruct: a dataset to finetune intermediate results \cite{luo2024graphinstruct}. 

%GraphGPT: GNN Transformer alignment \cite{tang2023graphgpt}. 

%GraphLLM \cite{chai2023graphllm} use an extra Graph Transformer 

\subsection{Algorithm Simulation with Transformers}
Recent theoretical investigations have shed light on the capability of the Transformer to simulate algorithms, a topic that has garnered considerable interest. %This discussion begins with discrete algorithms. 
%
%From a circuit complexity standpoint, 
From the view of discrete algorithms, Transformer models are likened to parallel circuits characterized by polynomial width and constant depth, which places them within the $\text{TC}^0$ complexity class (note that $\text{TC}^0 \subseteq \text{NC}^1 \subseteq P$). On the other hand, despite their impressive expressiveness, the Transformer is theoretically incapable of addressing a range of P-complete problems, including the testing of Context-Free Grammar Membership \cite{merrill2023parallelism}. However, the advent of chain-of-thought prompting has enabled the Transformer to sequentially simulate algorithms, thereby equipping them to tackle P-complete problems in domains such as arithmetic and decision-making \cite{feng2024towards}. The exploration then extends to continuous algorithms, where it has been demonstrated that the Transformer can approximate functions such as matrix inversion, Stochastic Gradient Descent (SGD), and power iterations \cite{giannou2023looped}. 
Our study specifically applies Transformer models to simulate path-finding algorithms, presenting evidence that their expressiveness is sufficient for such tasks (Theorem \ref{thm:expressive}). 
Nevertheless, the usage of autoregressive loss and gradient descent introduces certain limitations, which have not been studied in existing works.

\subsection{Mechansims of LLMs}
LLMs have demonstrated capabilities that exceed the theoretically predicted lower bounds of expressiveness. To demystify this paradox, numerous studies have employed experimental methodologies akin to those used in the physical and biological sciences. Their aim is to decode the mechanisms of LLMs. The foundational strategy is to generate controlled synthetic datasets to analyze how transformers (not necessarily LLMs) complete various tasks. Standard methods for this analysis include visualizing attention patterns to examine computational properties (such as locality and time invariance) and employing linear probing on the hidden states to determine the extent of learning. Given that the training data is synthetic and the ground-truth mappings are generally known, it becomes feasible to isolate the influence of various factors (e.g., prompting strategies, chain-of-thought reasoning, and data formatting). For example, a dataset designed for learning group operations, as detailed in \cite{zhang2022unveiling}, facilitates the exploration of how pretraining, data composition, and neural architecture influence reasoning tasks within LLMs. Similarly, the generation of synthetic context-free grammar (CFG) data, as described in \cite{allen2023physics}, enables training GPT-2 models, uncovering their capacity to learn dynamic programming algorithms for parsing CFGs. Synthetic datasets focusing on biographical knowledge, presented in \cite{zhu2023physics,allenzhu2023physics,zhu2024physics}, probe into the mechanisms of knowledge storage, retrieval, manipulation, and the implications of scaling laws. Moreover, the work in \cite{lee2023teaching} introduces synthetic datasets that aim to understand how smaller LLMs tackle basic arithmetic operations, e.g., addition, and examines the effects of few-shot prompting, pretraining, and model scaling \cite{lee2023teaching}. Our work builds upon these investigations by conducting controlled experiments with a path-finding dataset, thereby shedding light on the complexities and challenges of planning in language models.

%\section{Appendix / supplemental material}

\section{Proof of Theorem \ref{thm:expressive}}
\label{app:construction}
{\THMEXPRE*}

\begin{proof} For simplicity, we omit all layer normalizations in this construction. Suppose the input token sequence is ``$s\ t\ s\ u_1\ u_2\ \ldots\ u_k$'' with $k\geq 0$, where $s$ ($=u_0$) and 
    $t$ are the tokens of the source and target nodes, 
    respectively, and nodes $s,u_1,\cdots,u_k$ 
    form a path that can reach node $t$ in graph $\mathcal{G}$.
   % We want 
   Our objective is to construct a $1$-layer and $1$-head Transformer model 
   %to always
    that generates an out-neighbor $u_{k+1}$ of $u_k$ such that 
there exists at least one path from $u_{k+1}$
    to $t$ in $\mathcal{G}$.
    %Suppose that $u_1, \ldots, u_k$ are already correctly generated in this way, that is, they are on a valid path from $s$ to $t$ that pass through $u_1$ to $u_k$ in this order.

In essence, we 
%employ 
utilize the attention layer to attend the output 
%\emph{only} 
\emph{solely} 
to the target node $t$. 
%By this way,
This approach allows the distribution of next token $u_{k+1}$ to become a function of both the current node $u_{k}$ and the target node $t$ (as formulated in Section \ref{sec:prelim}). 
    Then, by integrating the adjacency matrix $\bm{A}^{\text{\rm true}}$ into the FFN layer and the reachability matrix $\bm{R}^{\text{\rm true}}$ into the matrix ${\bm W}^V$ in the attention layer, we extract row vectors 
    ${\bm R}^{\text{\rm true}}_{(t,:)}$ and ${\bm A}^{\text{\rm true}}_{(u_k,:)}$ from ${\bm R}^{\text{\rm true}}$ and ${\bm A}^{\text{\rm true}}$, respectively, corresponding to the target node $t$ and current node $u_k$. 
    %Specifically, ${\bm R}^{\text{\rm true}}_{(t,:)}$ and ${\bm A}^{\text{\rm true}}_{(u_k,:)}$ are stored by $\text{MHA}({\bm H}_0)$ and $\text{FFN}(\text{MHA}({\bm H}_0)+{\bm H}_0)$, respectively. 
    By selecting proper coefficients, we can %ignore the effect of the remaining term ${\bm H}_0$ in $\text{Transformer}({\bm H}_0)$ and only 
    let the output be the sum of ${\bm R}^{\text{\rm true}}_{(t,:)}$ and ${\bm A}^{\text{\rm true}}_{(u_k,:)}$. 
    Following the softmax layer, the non-negligible entries in the final vector correspond to the feasible next nodes.
    With this encoding, the Transformer serves as a simulator of Algorithm~\ref{alg:gt} with input $\bm{A} = \bm{A}^{\text{\rm true}}$ and $\bm{R} = \bm{R}^{\text{\rm true}}$.

    Following our notation in Section~\ref{sec:structure}, we adopt $d=|\mathcal{V}|+2$, $M=|\mathcal{V}|+1$ and $N=k+3$.
%    \wei{We need to use the same notation in the preliminary section here. For example, what is embedding dimension $d$? Is $d=|\mathcal{V}|+2$? If so, we should say
%    	it explicitly at the beginning. What is $M$, the vocabulary size? I guess it is also $|\mathcal{V}|+2$?
%    What is the $N$, the sequence size? Here we seem to talk about a sequence of length $k+3$, and predicting the next node. 
%	Is the structure general enough for any $k$?}\shi{This structure is general enough for any $k\geq 0$.}
    In the Transformer, there are $M = |\mathcal{V}|+1$ tokens representing the $|\mathcal{V}|$ nodes and the end-of-line ``$\backslash$n''.  
    Hence, the input tokens can be represented by the one-hot embedding matrix ${\bm U}\in \mathbb{R}^{N\times M}$. 
    We let ${\bm W}_t=\left(\begin{matrix}{\bm I}_{M\times M} \mid {\bm 0}_{M\times 1}\end{matrix}\right)\in \mathbb{R}^{M\times d}$ 
    and 
        $
        {\bm W}_p=\left(\begin{matrix}
                {\bm 0}_{(k+3)\times (|\mathcal{V}|+1)}\mid c_0\cdot {\bm e}_2\end{matrix}\right)\in\mathbb{R}^{N\times d}$,
        here
        ${\bm e}_2$ represents the second unit column vector of dimension $k+3$, $(A \mid B)$ is the notation for matrix concatenation by column,      
        and $c_0$ is a positive parameter to be decided.
    According to the definition of the Transformer, we now have a matrix ${\bm H}_0$ such that the first $|\mathcal{V}|+1$ columns 
    are the tokens of nodes in the sequence and the last column indicates the positions of the target node $t$.
    More specifically, we have\begin{align*}
        {\bm H}_0&=\left(\begin{matrix}
            {\bm e}^\top_{s}&0\\
            {\bm e}^\top_{t}&c_0\\
            {\bm e}^\top_{s}&0\\
            {\bm e}^\top_{u_1}&0\\
            \cdots&\cdots\\
            {\bm e}^\top_{u_k}&0
        \end{matrix}\right)\in\mathbb{R}^{N \times d},
    \end{align*}
    here ${\bm e}_u$ represents the one-hot token vector for node $u$ (with dimension $M = |\mathcal{V}| + 1$). 

    Then we construct the attention layer of the Transformer. We only have one head and let ${\bm W}^K=\left(\begin{matrix}
        {\bm 0}_{(|\mathcal{V}|+2)\times(|\mathcal{V}|+1)}\mid{\bm 1}_{(|\mathcal{V}|+2)\times 1}
    \end{matrix}\right)^\top\in\mathbb{R}^{d\times d}$ and ${\bm W}^Q=\sqrt{d}\cdot{\bm I}_{d \times d}$. Then we can compute ${\bm H}_0{\bm W}^K=\left(\begin{matrix}
        {\bm 0}_{(|\mathcal{V}|+2)\times1}\mid c_0\cdot {\bm 1}_{(|\mathcal{V}|+2)\times 1}\mid{\bm 0}_{(|\mathcal{V}|+2)\times (k+1)}
    \end{matrix}\right)^\top$, i.e., second rows are all $c_0$'s and other rows are all $0$'s,
	and ${\bm H}_0{\bm W}^Q=\sqrt{d}\cdot {\bm H}_0$. 

Therefore, 
\begin{small}\begin{align*}
\frac{({\bm H}_0{\bm W}^Q)({\bm H}_0{\bm W}^K)^\top}{\sqrt{d}}  &=\left(\begin{matrix}
            0&c_0&{\bm 0}_{1\times (k+1)}\\
           0&c_0^2+c_0&{\bm 0}_{1\times (k+1)}\\
            {\bm 0}_{(k+1)\times 1}&c_0\cdot {\bm 1}_{(k+1)\times 1}& {\bm 0}_{(k+1)\times (k+1)}\\
        \end{matrix}\right)
        \in \mathbb{R}^{N \times N}.
    \end{align*}\end{small}

And we can compute the first part of the attention layer as\begin{small}\begin{align*}
        \textbf{softmax}\left(\frac{({\bm H}_0{\bm W}^Q)({\bm H}_0{\bm W}^K)^\top}{\sqrt{d}}  \right)&=\left(\begin{matrix}
            \frac{1}{k+2+e^{c_0}}&\frac{e^{c_0}}{k+2+e^{c_0}}&\frac{1}{k+2+e^{c_0}}\cdot {\bm 1}_{1\times (k+1)}\\
           \frac{1}{k+2+e^{c_0^2+c_0}}&\frac{e^{c_0^2+c_0}}{k+2+e^{c_0^2+c_0}}&\frac{1}{k+2+e^{c_0^2+c_0}}\cdot {\bm 1}_{1\times (k+1)}\\
            \frac{1}{k+2+e^{c_0}}\cdot {\bm 1}_{(k+1)\times 1}&\frac{e^{c_0}}{k+2+e^{c_0}}\cdot {\bm 1}_{(k+1)\times 1}&\frac{1}{k+2+e^{c_0}}\cdot {\bm 1}_{(k+1)\times (k+1)}\\
        \end{matrix}\right)
        \in \mathbb{R}^{N \times N}.
    \end{align*}\end{small} 
    %When we choose 
    By setting $c_0\rightarrow +\infty$, we obtain:
    \begin{align*}
        \textbf{softmax}\left(\frac{({\bm H}_0{\bm W}^Q)({\bm H}_0{\bm W}^K)^\top}{\sqrt{d}}  \right) \to \left(\begin{matrix}
            0&1&{\bm 0}_{1\times (k+1)}\\
            \cdots&\cdots&\cdots\\
            0&1&{\bm 0}_{1\times (k+1)}
        \end{matrix}\right).
    \end{align*}
     Furthermore, we set ${\bm W}^V = \left(\begin{matrix}c_1\cdot {\bm R}^{\text{\rm true}}&{\bm 0}_{|\mathcal{V}|\times 2}\\{\bm 0}_{2\times |\mathcal{V}|}&{\bm 0}_{2\times 2}\end{matrix}\right)$, 
     where %${\bm R}^{\text{\rm true}}$ is the reachability matirx and 
     $c_1>0$ is also a parameter to be decided later. Then after the attention layer, we have a matrix as \begin{align*}
        \lim_{c_0\rightarrow+\infty}\text{MHA}({\bm H}_0)=c_1\cdot \left(\begin{matrix}
            {\bm R}^{\text{\rm true}}_{(t,:)}&0&0\\
            \cdots&\cdots&\cdots\\
            {\bm R}^{\text{\rm true}}_{(t,:)}&0&0
        \end{matrix}\right)\in \mathbb{R}^{N\times d}.
    \end{align*}

    Now we construct the feed-forward layer, which is a two-layer MLP. 
    
    For the first layer, the weight matrix ${\bm W}_1$ is set to be, \begin{align*}
        {\bm W}_1=\left(\begin{matrix}
            {\bm I}_{(|\mathcal{V}|+2)\times(|\mathcal{V}|+2)}&{\bm 0}_{(|\mathcal{V}|+2)\times 3(|\mathcal{V}|+2)}
        \end{matrix}\right)\in\mathbb{R}^{d \times 4d}.
    \end{align*}
    and the bias ${\bm b}_1 = -c_1 \cdot {\bm 1}_{4d\times 1}$, which %means 
    implies that ${\bm 1}_{N\times 1}{\bm b}_1^\top=-c_1\cdot {\bm 1}_{N\times 4d}$. %\siwei{do we require for all $j$? Or we only need to consider current node $u_k$?}
    When $c_0$ is large enough, the $(k+3)^{th}$ row of the matrix
    $\max\left({\bm 0},(\text{MHA}({\bm H}_0)+{\bm H}_0){\bm W}_1+{\bm 1}_{N\times 1}{\bm b}_1^\top\right)$ is $\max\left({\bm 0},c_1\cdot \left(\begin{matrix}{\bm R}_{(t,:)}\mid{\bm 0}_{1\times(3|\mathcal{V}|+8)}\end{matrix}\right)+\left(\begin{matrix}{\bm e}_{u_k}^\top \mid {\bm 0}_{1\times (3|\mathcal{V}|+7)}\end{matrix}\right)-c_1\cdot {\bm 1}_{1\times 4(|\mathcal{V}|+2)}\right)$. 
    Since $u_k$ can reach $t$, in $c_1\cdot \left(\begin{matrix}{\bm R}_{(t,:)}\mid{\bm 0}_{1\times(3|\mathcal{V}|+8)}\end{matrix}\right)+\left(\begin{matrix}{\bm e}_{u_k}^\top \mid {\bm 0}_{1\times (3|\mathcal{V}|+7)}\end{matrix}\right)$, only the entry for node $u_k$ is $c_1+1$ while all other entries are $0$ or $c_1$.
    Therefore, the $(k+3)^{th}$ row of the matrix
    $\max\left({\bm 0},(\text{MHA}({\bm H}_0)+{\bm H}_0){\bm W}_1+{\bm 1}_{N\times 1}{\bm b}_1^\top\right)$ can be arbitrarily close to 
    $\left(\begin{matrix}{\bm e}_{u_k}^\top \mid {\bm 0}_{1\times (3|\mathcal{V}|+7)}\end{matrix}\right)$. Here ${\bm e}_u$ represents the one-hot token vector for node $u$ (with dimension $M = |\mathcal{V}| + 1$).

   % we claim that for every $j=0,1,\ldots, k$, the $(j+3)$-th row of the matrix
    %$\max\left({\bm 0},(\text{MHA}({\bm H}_0)+{\bm H}_0){\bm W}_1+{\bm 1}_{N\times 1}{\bm b}_1^\top\right)$ %can be arbitrarily close to 
    %$\left(\begin{matrix}{\bm e}_{u_j}^\top \mid {\bm 0}_{1\times (3|\mathcal{V}|+7)}\end{matrix}\right)$. 
    
    %This is because the $(j+3)$-th row is 
    %$\max\left({\bm 0},c_1\cdot \left(\begin{matrix}{\bm R}_{(t,:)}\mid{\bm 0}_{1\times(3|\mathcal{V}|+8)}\end{matrix}\right)+\left(\begin{matrix}{\bm e}_{u_j}^\top \mid {\bm 0}_{1\times (3|\mathcal{V}|+7)}\end{matrix}\right)-c_1\cdot {\bm 1}_{1\times 4(|\mathcal{V}|+2)}\right)$. 
    %Since $u_j$ can reach $t$, in $c_1\cdot \left(\begin{matrix}{\bm R}_{(t,:)}\mid{\bm 0}_{1\times(3|\mathcal{V}|+8)}\end{matrix}\right)+\left(\begin{matrix}{\bm e}_{u_j}^\top \mid {\bm 0}_{1\times (3|\mathcal{V}|+7)}\end{matrix}\right)$, only the entry for node $u_j$ is $c_1+1$ while all other entries are $0$ or $c_1$.
    For the second layer, we set \begin{align*}
        {\bm W}_2=\left(\begin{matrix}
            c_2\cdot {\bm A}&{\bm 0}_{|\mathcal{V}|\times 2}\\
            {\bm 0}_{(3|\mathcal{V}|+8)\times |\mathcal{V}|}&{\bm 0}_{(3|\mathcal{V}|+8)\times 2}
        \end{matrix}\right)\in\mathbb{R}^{4d\times d},
    \end{align*}
    where $c_2$ is a positive parameter to be decided, and ${\bm b}_2={\bm 0}$.
    By this way, we have 
    %When $c_0$ is large enough, the $(k+3)$-th row of the result matrix $\text{FFN}(\text{MHA}({\bm H}_0)+{\bm H}_0)$ approaches the following vector:
    \begin{align*}
    \lim_{c_0 \to \infty}(\text{FFN}(\text{MHA}({\bm H}_0)+{\bm H}_0))_{(k+3,:)} \to \left(\begin{matrix}
        c_2\cdot {\bm A}_{(u_k,:)} \mid {\bm 0}_{1\times 2}
    \end{matrix}\right)\in\mathbb{R}^{|\mathcal{V}|+2}.
    \end{align*}
    Therefore, 
    \begin{align*}
    \lim_{c_0 \to \infty}({\bm H}_1)_{(k+3,:)} \to \left(\begin{matrix}
        c_1\cdot {\bm R}_{(t,:)}+c_2\cdot  {\bm A}_{(u_j,:)} \mid {\bm 0}_{1\times 2}
    \end{matrix}\right)+\left(\begin{matrix}{\bm e}_{u_k}\mid 0\end{matrix}\right)\in\mathbb{R}^{|\mathcal{V}|+2},
    \end{align*}
    where ${\bm e}_u$ represents the one-hot token vector for node $u$ (with dimension $M = |\mathcal{V}| + 1$). 
    %the row vector $({\bm H}_1)_{(k+3,:)}$ equals to $\left(\begin{matrix}
    %    c_1\cdot {\bm R}_{(t,:)}+c_2\cdot  {\bm A}_{(u_j,:)} \mid c_3\mid0
    %\end{matrix}\right)+\left(\begin{matrix}{\bm e}_{u_j}\mid0\end{matrix}\right)\in\mathbb{R}^{|\mathcal{V}|+2}$.

    %\siwei{In the description of this theorem, it seems that we do not consider EOS.}
    Then we fix $c_1=c_2$ and let them be large enough.
    %When $u_j\neq t$, 
    In this case, the dominant entries in $({\bm H}_1)_{(k+3,:)}$ represent the nodes that are both the out-neighbor of $u_j$ and reachable to $t$, since those entries will
    have the value of $2c_1$ while other entries are at most $c_1+1$.
    This means that $({\bm H}_1)_{(k+3,:)}$ can correctly indicates the next node $u_{k+1}$. 
    %In particular, for $j=k$, the $({\bm H}_1)_{(j+3,:)}$ correctly indicates the next node $u_{k+1}$.
    %When $u_k=t$, ${\bm R}_{(t,:)}$ and ${\bm A}_{(u_k,:)}$ do not have a common nonzero entry since no node can be both an out-neighbor of $t$ and still reachable to $t$ in a DAG.
    %Therefore, when $u_k=t$, the $(|\mathcal{V}|+1)$-th entry of vector $({\bm H}_1)_{(k+3,:)}$ with value $c_3$ dominates other entries with value at most $\frac{2}{3}c_3$.
    %This entry corresponds to the end-of-line token `$\backslash$n'.
    %
    Specifically, let ${\bm W}_o=\left(\begin{matrix}
	{\bm I}_{(|\mathcal{V}|+1)\times(|\mathcal{V}|+1)}\mid{\bm 0}_{(|\mathcal{V}|+1)\times 1}
	\end{matrix}\right)^\top\in \mathbb{R}^{d \times M}.$
    Then the final output approaches the following vector\begin{align*}
        \lim_{c_0,c_1=c_2 \to \infty}\hat{\bm{u}}_{k+1}&=\lim_{c_0,c_1=c_2 \to \infty}\text{softmax}(({\bm H}_1)_{(k+3,:)}{\bm W}_o)\\
        &=%\begin{cases}
            %{\bm e}_{|\mathcal{V}|+1}^\top&u_k=t\\
            \frac{1}{C} \cdot \left(\begin{matrix}
               \mathbb{I}[{\bm A}_{(u_k,1)}=1\wedge {\bm R}_{(t,1)}=1],\cdots,\mathbb{I}[{\bm A}_{(u_k,|\mathcal{V}|)}=1\wedge {\bm R}_{(t,|\mathcal{V}|)}=1],0
            \end{matrix}
            \right), % &u_k\neq t
        %\end{cases},
    \end{align*}
    where $C$ is the number of nodes that are both the out-neighbor of $u_k$ and reachable to $t$.
    Thus, this encoding guarantees that %this is exactly the correct output of the next node. 
    %Hence, 
    for any $\varepsilon>0$ and $Q>0$, we can always construct a $1$-layer, $1$-head, and $(|\mathcal{V}|+2)$-embedding-size Transformer that provides the correct next token with probability %larger than
    at least $1-\frac{\varepsilon}{2Q}$ by selecting large enough parameters $c_0,c_1,c_2$.

    Then we prove that there exists a $Q$ such that this Transformer can output a correct path for every valid source and target node pair with probability at least $1-\varepsilon$. Suppose we can output all the nodes that are both the out-neighbor of the current node and reachable to the target node with the same probability in each round without any error. Then, whatever the current node is, there is a probability of at least $\frac{1}{|\mathcal{V}|^{|\mathcal{V}|}}$ that the target node is reached within the next $|\mathcal{V}|$ generated nodes. Therefore, %since the length of the path is no more than $|\mathcal{V}|$, 
    the target node is reached in $c_3 \cdot |\mathcal{V}|$ steps with probability at least $1-\left(1-\frac{1}{|\mathcal{V}|^{|\mathcal{V}|}}\right)^{c_3}$, where $c_3 \in \mathbb{N}$ is a positive integer. We let $c_3 = \log_{1-\frac{1}{|\mathcal{V}|^{|\mathcal{V}|}}}\frac{\epsilon}{2}$ and $Q = c_3 \cdot |\mathcal{V}|$. Then, according to the Union bound, the Transformer can output a correct path in $Q$ steps with an error rate of at most $\frac{\varepsilon}{2Q} \cdot Q + \frac{\varepsilon}{2} = \varepsilon$. 
    
    Finally, there are two different rules (other than output a correct next node): i) when the input sequence is only ``$s$ $t$'', the prediction of the next token should be the source node $s$; ii) when the input sequence is ``$s$ $t$ $s$ $a$ $b$ $c$ $t$'', the prediction of the next token should be $\backslash$n. Case i) can be solved using the Transformer architecture utilizing the position information and attention to the first position; and case ii) can be solved by using the Transformer architecture utilizing the position information and attention to the second position.
To maintain focus on the main construction corresponding to Algorithm~\ref{alg:gt}, we omit the detailed construction for these two boundary cases.
\end{proof}
\section{Proof of Theorem \ref{Thm_1}}
{\THMGD*}

\begin{proof}
We only prove the first part of this theorem, since the proof of the second part is almost %the 
%same.
identical.

By the definition of the cross-entropy loss in Eq.\eqref{eq:ce_loss}, and the prediction weight vector in Eq.\eqref{eq:predweight} for our simplified model, 
	the total cross-entropy loss of the model (with matrices $\bm{W}^{M}$, $\bm{W}^{V}$) is 
%\wei{It is not very straightforward to go from Eq.\eqref{eq:ce_loss} to the first equation below. Perhaps some intermediate steps would be helpful to show the derivation, in particular,
%	the derivation of $\hat{\bm{u}}_{(n+1),j}$.}
\begin{eqnarray*}
    \ell(\mathcal{D}) &=& -  \sum_{\bm{u}\in \mathcal{D}} \sum_{n\ge3} \sum_{k} \bm{U}_{(n+1,k)} \log \hat{\bm{u}}_{(n+1),k}\\
    &=& -  \sum_{\bm{u}\in \mathcal{D}} \sum_{n\ge3} \sum_{k } \bm{U}_{(n+1,k)} \log {\exp(\bm{W}^{M}_{(u_n,k)}+\bm{W}^{V}_{(u_2,k)}) \over \sum_{\ell} \exp(\bm{W}^{M}_{(u_n,\ell)}+\bm{W}^{V}_{(u_2,\ell)})}\\
    &=& -  \sum_{\bm{u}\in \mathcal{D}} \sum_{n\ge3}\sum_{k} \bm{U}_{(n+1,k)} \sum_{i,j} \I[u_n = i, u_2 = j]\log {\exp(\bm{W}^{M}_{(i,k)}+\bm{W}^{V}_{(j,k)}) \over \sum_{\ell} \exp(\bm{W}^{M}_{(i,\ell)}+\bm{W}^{V}_{(j,\ell)})}\\
    &=& - \sum_{i,j,k} N_{i,j,k} \log {\exp(\bm{W}^{M}_{(i,k)}+\bm{W}^{V}_{(j,k)}) \over \sum_{\ell} \exp(\bm{W}^{M}_{(i,\ell)}+\bm{W}^{V}_{(j,\ell)})}\\
    &=& - \sum_{i,j,k} N_{i,j,k} (\bm{W}^{M}_{(i,k)} + \bm{W}^{V}_{(j,k)}) + \sum_{i,j,k} N_{i,j,k} \log \left(\sum_{\ell} \exp(\bm{W}^{M}_{(i,\ell)}+\bm{W}^{V}_{(j,\ell)})\right)\\
    &=& - \sum_{i,j,k} N_{i,j,k} (\bm{W}^{M}_{(i,k)} + \bm{W}^{V}_{(j,k)}) + \sum_{i,j} N_{i,j} \log \left(\sum_{\ell} \exp(\bm{W}^{M}_{(i,\ell)}+\bm{W}^{V}_{(j,\ell)})\right).
\end{eqnarray*}
Then we have that
\begin{equation}\label{Eq_1_copy}
    {\partial \ell(\mathcal{D}) \over \partial \bm{W}^{M}_{(i,k)}} = - \sum_j N_{i,j,k} + 
    \sum_{j} N_{i,j} {\exp(\bm{W}^{M}_{(i,k)}+\bm{W}^{V}_{(j,k)}) \over \sum_{\ell} \exp(\bm{W}^{M}_{(i,\ell)}+\bm{W}^{V}_{(j,\ell)})}. %\cdot \sum_{j} N_{i,j}.
    % = \sum_j \left(N_{i,j} {\exp(M_{(i,k)}+V_{(j,k)}) \over \sum_{\ell} \exp(M_{i,\ell}+V_{j,\ell})} - N_{i,j,k}\right)
\end{equation}

In case i), $\sum_{j} N_{i,j} = 0$ implies that $\sum_j N_{i,j,k} = 0$. Hence ${\partial \ell(\mathcal{D}) \over \partial \bm{W}^{M}_{(i,k)}}$ is always 0.

In case ii), $\sum_{j} N_{i,j} > 0$ implies that the second term in Eq. \eqref{Eq_1_copy} is positive, while $\sum_j N_{i,j,k} = 0$ implies that the first term in Eq. \eqref{Eq_1_copy} is 0. Hence ${\partial \ell(\mathcal{D}) \over \partial \bm{W}^{M}_{(i,k)}}$ is always positive.

In case iii), when $\sum_{j} N_{i,j} > 0$ and $\bm{W}^{M}_{(i,k)}$ converges to $-\infty$, then the second term in Eq. \eqref{Eq_1_copy} converges to 0, and it is smaller than $\sum_j N_{i,j,k} > 0$. Hence, ${ \partial \ell(\mathcal{D}) \over \partial \bm{W}^{M}_{(i,k)}}$ is negative when $\bm{W}^{M}_{(i,k)}$ converges to $-\infty$.
\end{proof}

\section{The Reason of Choosing $W^{M'}$ and $W^{V'}$}\label{sec:rea_matrix}

Note that in the Transformer layer, the output can be written as\footnote{For simplicity, though the layer normalizations $\text{LN}_1$,$\text{LN}_2$ and $\text{LN}_t$ are used in our experiments, we omit them in the equations in this section.}
\begin{equation*}
    \text{FFN}\left(\textbf{softmax}\left(\frac{\bm{Q}\bm{K}^\top}{\sqrt{d_k}}  \right)\bm{V} + \bm{X}\right) + \textbf{softmax}\left(\frac{\bm{Q}\bm{K}^\top}{\sqrt{d_k}}  \right)\bm{V} + \bm{X}.
\end{equation*}
Also noting that we have verified that the attention is concentrated at the second token, then we let $\bm{X}_2 = \bm{U}_{(2,:)}\bm{W}_t$, representing the token embedding of the target node, and $\bm{X}_n = \bm{U}_{(n,:)}\bm{W}_t$, representing the token embedding of the current node. 

\vspace{1cm}

Then we know that
\begin{equation*}
\hat{\bm{u}}_{(n+1)} \approx \left(\text{FFN}\left(\bm{X}_2 \bm{W}^V + \bm{X}_n\right) + \bm{X}_2 \bm{W}^V + \bm{X}_n\right) \bm{W}_o.
\end{equation*}

\begin{table}[ht]
	\centering
	\caption{Cosine Similarity of $\text{FFN}\left(\bm{X}_2 \bm{W}^V + \bm{X}_n\right)\bm{W}_o$ and  $\text{FFN}\left(\bm{X}_2 \bm{W}^V \right)\bm{W}_o$ + $\text{FFN}\left(\bm{X}_n\right)\bm{W}_o$}\label{Tab:1}
	\begin{tabular}{|c|c|c|c|c|c|}
		\hline
		\textbf{Graph} & 100 Nodes & 200 Nodes & 300 Nodes & 400 Nodes & 500 Nodes \\ \hline
		\textbf{Average Cosine Similarity}     & 0.926     & 0.924    & 0.901 & 0.870 & 0.889   \\ \hline
	\end{tabular}
\end{table}

It is straightforward that $\bm{X}_n\bm{W}_o$ contains the information of the current node, and $\bm{X}_2 \bm{W}^V \bm{W}_o$ contains the information of the target node. As for $\text{FFN}\left(\bm{X}_2 \bm{W}^V + \bm{X}_n\right)\bm{W}_o$, we choose to use its linear approximation as $\text{FFN}\left(\bm{X}_2 \bm{W}^V + \bm{X}_n\right)\bm{W}_o \approx \text{FFN}\left(\bm{X}_2 \bm{W}^V \right)\bm{W}_o + \text{FFN}\left(\bm{X}_n\right)\bm{W}_o $. As shown in Table \ref{Tab:1} (which takes average over all possible $\bm{X}_2$'s and $\bm{X}_n$'s), this is a good approximation. 
Then we can treat $\text{FFN}\left(\bm{X}_2 \bm{W}^V \right)\bm{W}_o$ as the information of the target node, and $\text{FFN}\left(\bm{X}_n \right)\bm{W}_o$ as the information of the current node.

Because of this, we let $\bm{W}^{M'}$ be the matrix whose $i$-th row is $\text{FFN}\left({\bm e}_i^\top \bm{W}_t\right)\bm{W}_o + ({\bm e}_i^\top \bm{W}_t)\bm{W}_o$, where ${\bm e}_i$ represents the one-hot column vector for node $i$ (with dimension $M = |\mathcal{V}| + 1$). 
Note that in the simplified Transformer model of Theorem \ref{Thm_1}, there is no $\bm{X}_n\bm{W}_o$ term, and $\bm{W}^{M'}$ is the same as matrix $\bm{W}^M$.
Similarly, we let $\bm{W}^{V'}$ be the matrix whose $i^{th}$ row is $({\bm e}_i^\top\bm{W}_t) \bm{W}^V\bm{W}_o + \text{FFN}\left(({\bm e}_i^\top\bm{W}_t) \bm{W}^V \right)\bm{W}_o$, where ${\bm e}_i$ represents the one-hot column vector for node $i$ (with dimension $M = |\mathcal{V}| + 1$). 
Note that in the simplified Transformer model of Theorem \ref{Thm_1}, there is no $\text{FFN}\left(\bm{X}_2 \bm{W}^V \right)\bm{W}_o$ term, and $\bm{W}^{V'}$ is the same as matrix $\bm{W}^V$.

%Note that in the simplified Transformer model of Section \ref{sec:theoretical}, there is no $\bm{X}_n\bm{W}_o$ term and $\text{FFN}\left(\bm{X}_2 \bm{W}^V \right)\bm{W}_o$ term. In this case
%is the same as matrix $\bm{W}^M$.

\vspace{-0.2cm}
\section{Complete Experimental Results on the Synthetic Datasets}\label{sec:complete_exp}

%\subsection{Accuracy on Test Dataset}

%We first show a general accuracy result in Figure \ref{fig1:five_images}. In these results, we test Transformer architecture from 1 layer and 1 head to 6 layers and 6 heads, and consider different graph size (number of nodes) from 100 to 500.

\begin{figure}[t]
    \centering
    % Top row with three images
    \begin{subfigure}[100 Nodes]{
        \centering
        \includegraphics[width=0.30\linewidth]{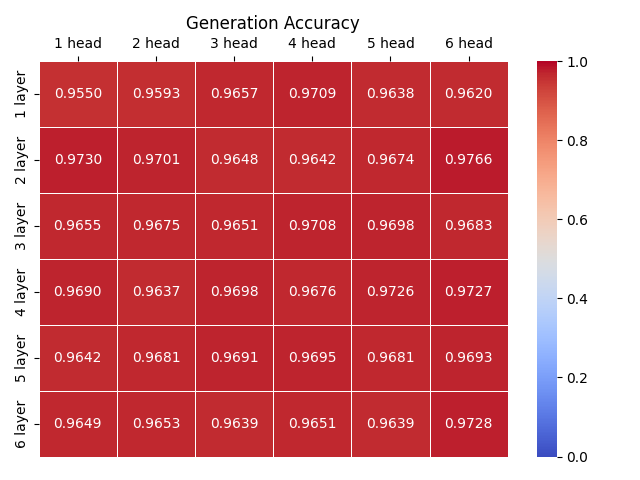}
        \label{fig1ap:image1}}
    \end{subfigure}
    \begin{subfigure}[200 Nodes]{
        \centering
        \includegraphics[width=0.30\linewidth]{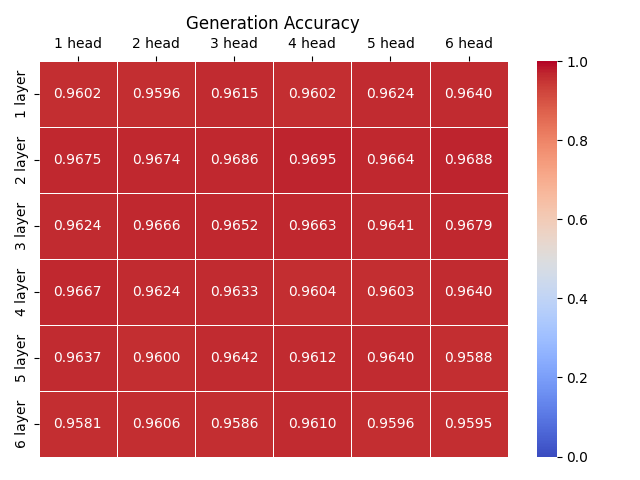}
        \label{fig1ap:image2}}
    \end{subfigure}
    \begin{subfigure}[300 Nodes]{
        \centering
        \includegraphics[width=0.30\linewidth]{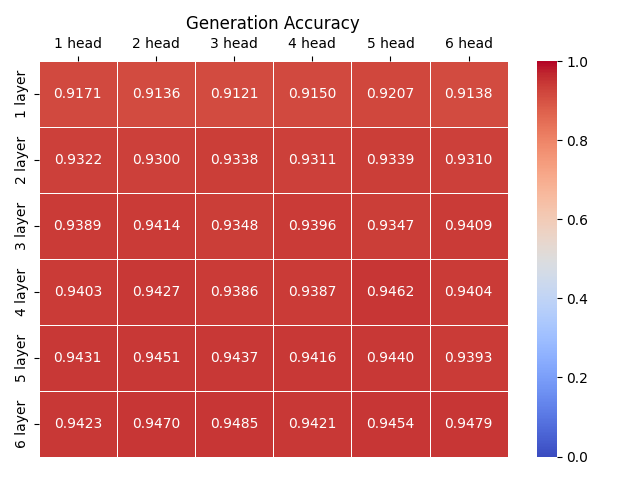}
        \label{fig1ap:image3}}
    \end{subfigure}

    %\vspace{1cm} % Space between the two rows

    % Bottom row with two images
    \begin{subfigure}[400 Nodes]{
        \centering
        \includegraphics[width=0.30\linewidth]{Fig/E2/Embd=120_Temp=1_Nodes=400_CBar.png}
        \label{fig1ap:image4}}
    \end{subfigure}
    \begin{subfigure}[500 Nodes]{
        \centering
        \includegraphics[width=0.30\linewidth]{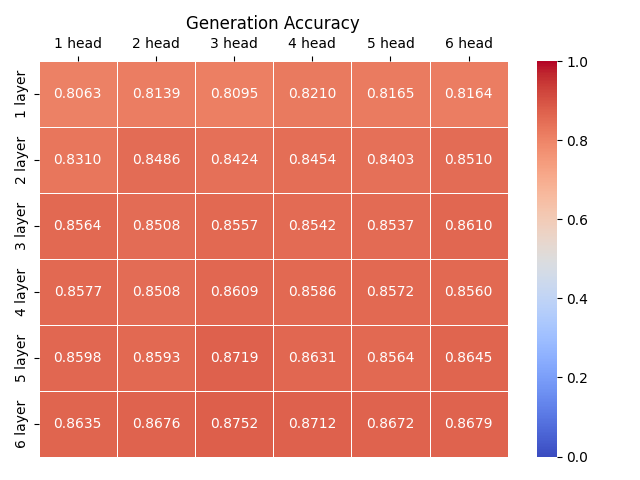}
        \label{fig1ap:image5}}
    \end{subfigure}

    \caption{Accuracy on the test datasets with embedding size $d= 120$. }
    \label{fig1ap:five_images}
\end{figure}

%For example, a 6-layer transformer has more than 4 times parameter i

\begin{figure}[!h]
	\centering
	% Top row with three images
	\begin{subfigure}[100 Nodes]{
			\centering
			\includegraphics[width=0.30\linewidth]{Fig/E3/Att_100_1_1_120_Layer1_Head1.png}
			\label{fig2ap:image1}}
	\end{subfigure}
	\begin{subfigure}[200 Nodes]{
			\centering
			\includegraphics[width=0.30\linewidth]{Fig/E3/Att_200_1_1_120_Layer1_Head1.png}
			\label{fig2ap:image2}}
	\end{subfigure}
	\begin{subfigure}[300 Nodes]{
			\centering
			\includegraphics[width=0.30\linewidth]{Fig/E3/Att_300_1_1_120_Layer1_Head1.png}
			\label{fig2ap:image3}}
	\end{subfigure}
	
	%\vspace{1cm} % Space between the two rows
	
	% Bottom row with two images
	\begin{subfigure}[400 Nodes]{
			\centering
			\includegraphics[width=0.30\linewidth]{Fig/E3/Att_400_1_1_120_Layer1_Head1.png}
			\label{fig2ap:image4}}
	\end{subfigure}
	\begin{subfigure}[500 Nodes]{
			\centering
			\includegraphics[width=0.30\linewidth]{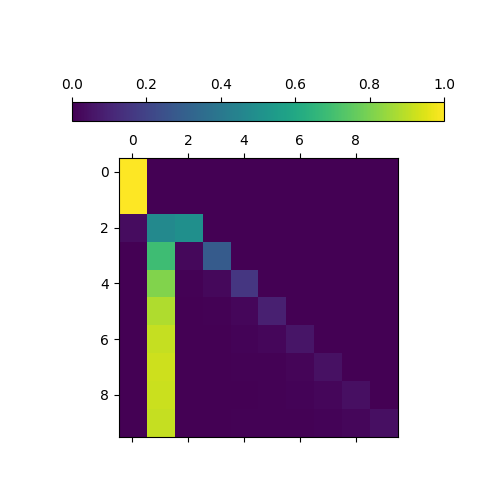}
			\label{fig2ap:image5}}
	\end{subfigure}
	
	\caption{The average attention in 1-layer and 1-head Transformers.}
	\label{fig2ap:five_images}
\end{figure}

\begin{figure}[!h]
	\centering
	% Top row with three images
	\begin{subfigure}[100 Nodes]{
			\centering
			\includegraphics[width=0.28\linewidth]{Fig/E4/New_MLP_weight_1_1_120_100_20.png}
			\label{fig3ap:image1}}
	\end{subfigure}
	\begin{subfigure}[200 Nodes]{
			\centering
			\includegraphics[width=0.28\linewidth]{Fig/E4/New_MLP_weight_1_1_120_200_20.png}
			\label{fig3ap:image2}}
	\end{subfigure}
	\begin{subfigure}[300 Nodes]{
			\centering
			\includegraphics[width=0.28\linewidth]{Fig/E4/New_MLP_weight_1_1_120_300_20.png}
			\label{fig3ap:image3}}
	\end{subfigure}
	
	%\vspace{1cm} % Space between the two rows
	
	% Bottom row with two images
	\begin{subfigure}[400 Nodes]{
			\centering
			\includegraphics[width=0.28\linewidth]{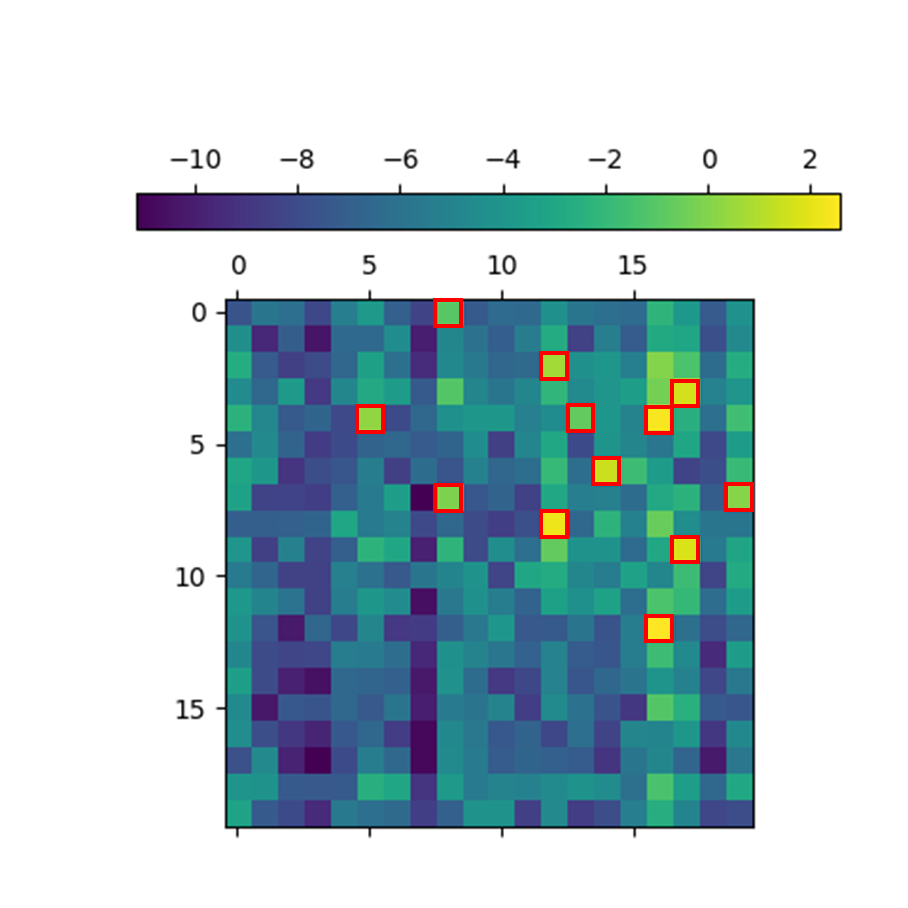}
			\label{fig3ap:image4}}
	\end{subfigure}
	\begin{subfigure}[500 Nodes]{
			\centering
			\includegraphics[width=0.28\linewidth]{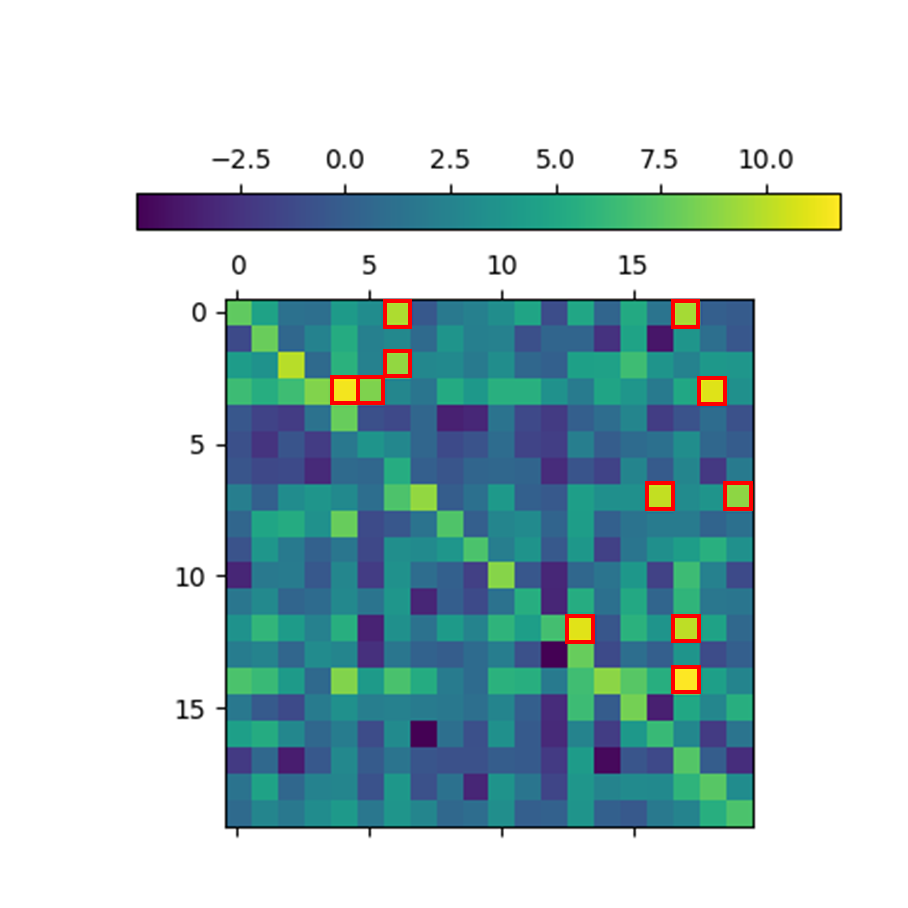}
			\label{fig3ap:image5}}
	\end{subfigure}
        %\begin{subfigure}[Average Weight Gap]{
	%		\centering
	%		\includegraphics[width=0.36\linewidth]{Fig/E4/MLP_AVE_GAP_Embd=120.png}
	%		\label{fig3ap:image6}}
	%\end{subfigure}
	
	\caption{The first 20 rows and columns of $\bm{W}^{M'}$ matrix in the 1-layer and 1-head Transformers (the red boxes correspond to $1$'s in the adjacency matrix $\bm{A}$).
			}
	\label{fig3ap:five_images}
\end{figure}

\begin{figure}[!h]
    \centering
    % Top row with three images
    \begin{subfigure}[100 Nodes]{
        \centering
        \includegraphics[width=0.30\linewidth]{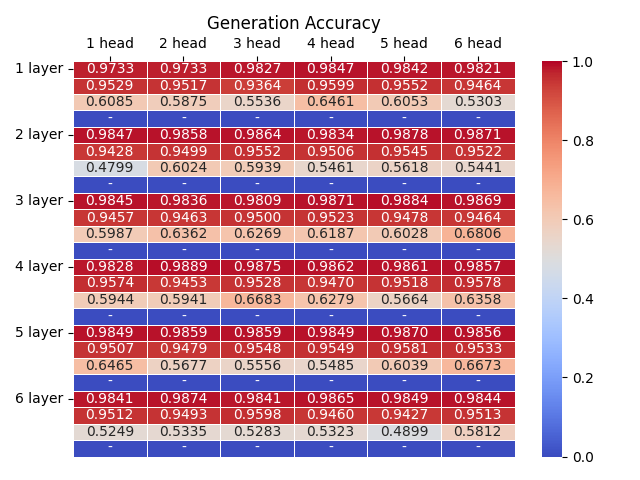}
        \label{fig4ap:image1}}
    \end{subfigure}
    \begin{subfigure}[200 Nodes]{
        \centering
        \includegraphics[width=0.30\linewidth]{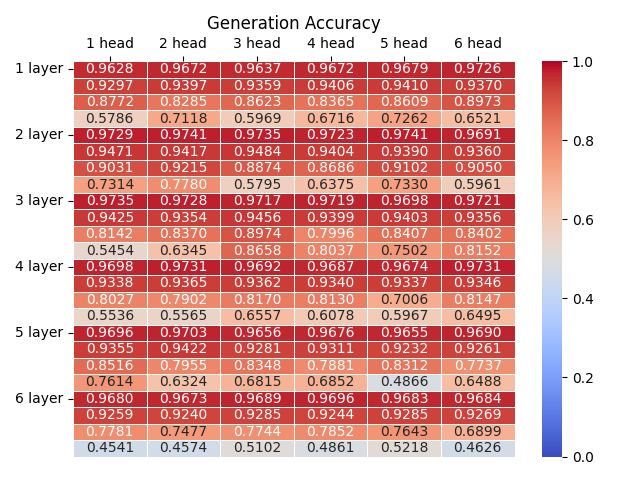}
        \label{fig4ap:image2}}
    \end{subfigure}
    \begin{subfigure}[300 Nodes]{
        \centering
        \includegraphics[width=0.30\linewidth]{Fig/E5/DF_Embd=120_Temp=1_Nodes=300_Cbar.png}
        \label{fig4ap:image3}}
    \end{subfigure}
    
    %\vspace{1cm} % Space between the two rows
    % Bottom row with two images
    \begin{subfigure}[400 Nodes]{
        \centering
        \includegraphics[width=0.30\linewidth]{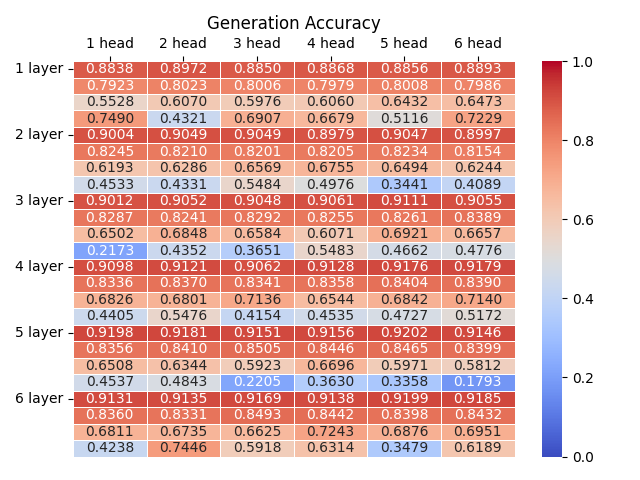}
        \label{fig4ap:image4}}
    \end{subfigure}
    \begin{subfigure}[500 Nodes]{
        \centering
        \includegraphics[width=0.30\linewidth]{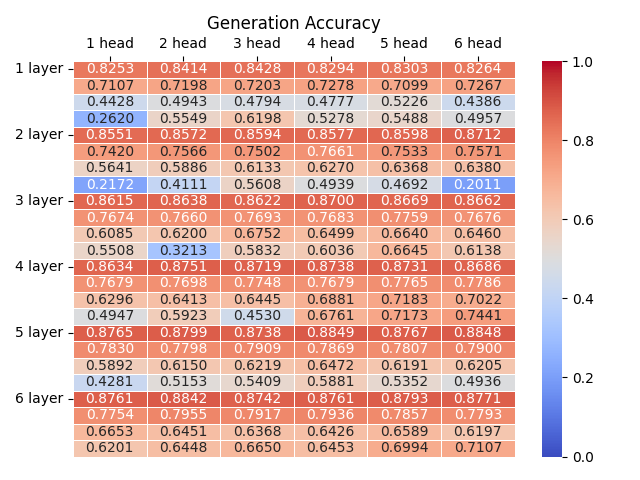}
        \label{fig4ap:image5}}
    \end{subfigure}

    \caption{The accuracy for $(s,t)$'s with different degrees. }
    \label{fig4ap:five_images}
\end{figure}

%\subsection{Attention}

%We now look inside the Transformer models, and try to find more evidence that our theoretical analysis reflects the reality.
%
%In our theoretical analysis, we assume that the attention is fixed on the target node.
%we first assume that the attention is fixed at the target node. 
%Is it true for the Transformer models learned from read data?
%The corresponding results are shown in Figure \ref{fig2:five_images}.
\vspace{-0.2cm}
In this section, we state the complete experimental results on the synthetic datasets. All these results are conducted on a single A100 GPU. 
\begin{itemize}
    \item The accuracy results on all these tests are 
%reported
presented in Figure \ref{fig1ap:five_images}.
\item The attention results on all the 1-layer and 1-head Transformers are 
%reported
presented in Figure \ref{fig2ap:five_images}.
\item The results of $\bm{W}^{M'}$'s are shown in Figure \ref{fig3ap:five_images}. 

\item The accuracy of the Transformers on the $(s,t)$ pairs of the four different categories are shown in Figure \ref{fig4ap:five_images}. 

\end{itemize}

As we can see, all these results are consistent with our conclusions.

%results are obtained by looking into the attention mechanism of the five 1-layer and 1-head Transformer models, 
%	and showing the average (taking on the test dataset) matrix of $\textbf{softmax}\left(\frac{{\bm{Q}\bm{K}^\top}}{\sqrt{d_k}}  \right)$, 
%	of which the $n^{th}$ row represents the attention vector for predicting the $(n+1)^{th}$ token.

%Note that the second column in these figures represents the attention weights on the second token, which corresponding to the target node in our test data.
%We can see that, when predicting the next tokens, almost all the attention weights are concentrated on this column, especially for those models
%	with higher accuracy (Figure \ref{fig2:image1} for $n=100$ and Figure \ref{fig2:image2} for $n=200$). 
%This demonstrates that indeed the Transformer model learns the correct attention for the path-finding task, and our assumption on the attention for the theoretical analysis is reasonable.

\section{Path-planning in Blocksworld}
\label{sec:blocksworld}
To further validate the theoretical results in Section~\ref{sec:overview} and the practicability of the proposed path-finding task, we consider Blocksworld benchmark~\cite{valmeekam2023planning}. Blocksworld is a scenario consisting of a set of blocks identified by different colors. The blocks are either placed on table or on top of another block and the task is to plan a block manipulation from the source state to the target state. 

We formulate Blocksworld as a path-finding task. 
Here we construct a graph $G_{BW}$ for the case with $4$ blocks, where each node represents a state of the blocks. 
For example, node $0$ refers to the state that ``the red block is on top of the blue block, the blue block is on top of the orange block, the orange block is on top of the yellow block, and the yellow block is on the table''. 
$G_{BW}$ is a directed graph with $73$ nodes, and the adjacency matrix of $G_{BW}$ is presented in Figure~\ref{fig:BW-adjacency-real}. %$G_{BW}$ has distinct structure. 
%%Specifically, the states corresponding to the first $24$ nodes are the states where the four blocks are stacked one on top of another in a single stack, and thus each of these nodes only has one out-neighbor
%%	corresponding to removing the top block and putting it on the table. 
%%The last node refers to the state where all blocks are on the table,  so it has 12 out-neighbors, corresponding to the $12$ states where three blocks are on the table and the fourth block is on top of one of the three blocks.

%is connected to the $12$ states where three blocks are on the table and the fourth block in on top of one of the three blocks.

In the original Blocksworld task, the answer is a sequence of actions, which is equivalent to the notion of edges in $G_{BW}$. 
We reformulated it to let the model output a path from the given source state to the given target state, only consisting of the nodes.
This can be seen as a simplified version and a pure planning task.
We randomly select $80\%$ of all node pairs for training and the rest $20\%$ for testing, and generate $50000$ training sequences in the same format as introduced in Section~\ref{sec:prelim}. 
We mainly use Transformers with $1$ layer and $1$ head for the convenience of visualization.
\begin{figure}[t]
    \centering
    \begin{subfigure}[Adjacency Matrix of $G_{BW}$]{
        \centering
        \includegraphics[trim=5 5 5 5, clip, width=0.3\linewidth]{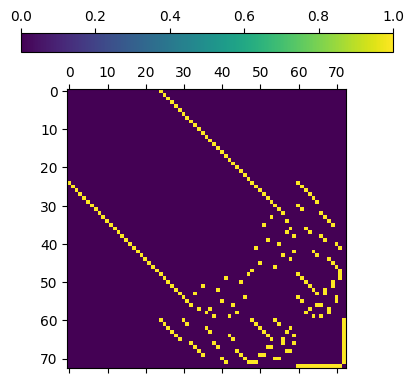}
        \label{fig:BW-adjacency-real}}
    \end{subfigure}
    \begin{subfigure}[The $\bm{W}^{M'}$ Matrix]{
        \centering
        \includegraphics[trim=40 25 17 40, clip, width=0.3\linewidth]{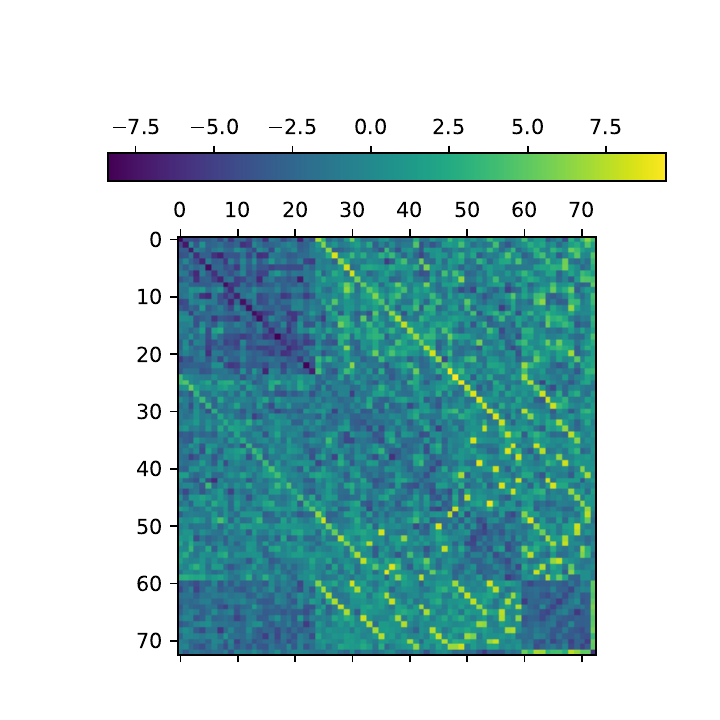}
        \label{fig:BW-adjacency-learned}}
    \end{subfigure}
    \begin{subfigure}[The Average Attention]{
        \centering
        \includegraphics[trim=40 25 17 40, clip, width=0.3\linewidth]{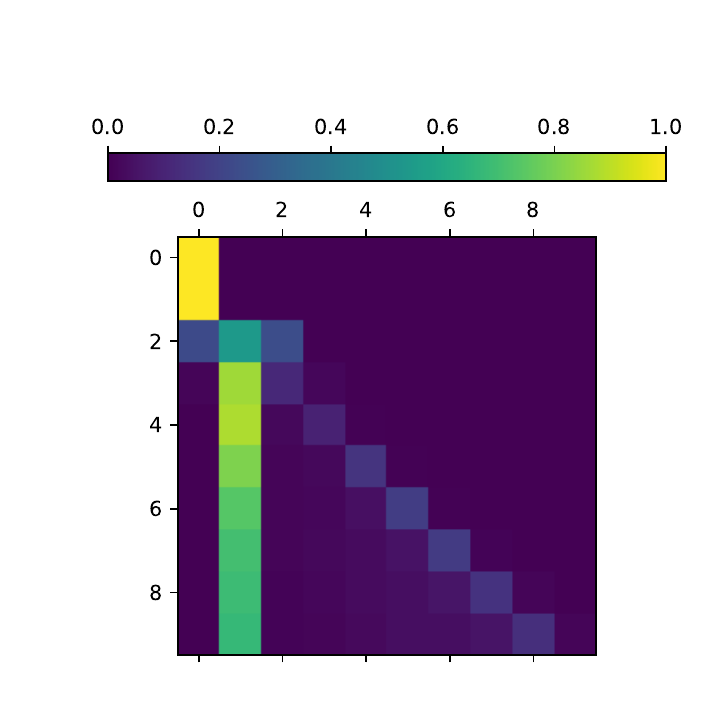}
        \label{fig:BW-attention}}
    \end{subfigure}
    \begin{subfigure}[Accuracy Curve]{
        \centering
        \includegraphics[width=0.35\linewidth]{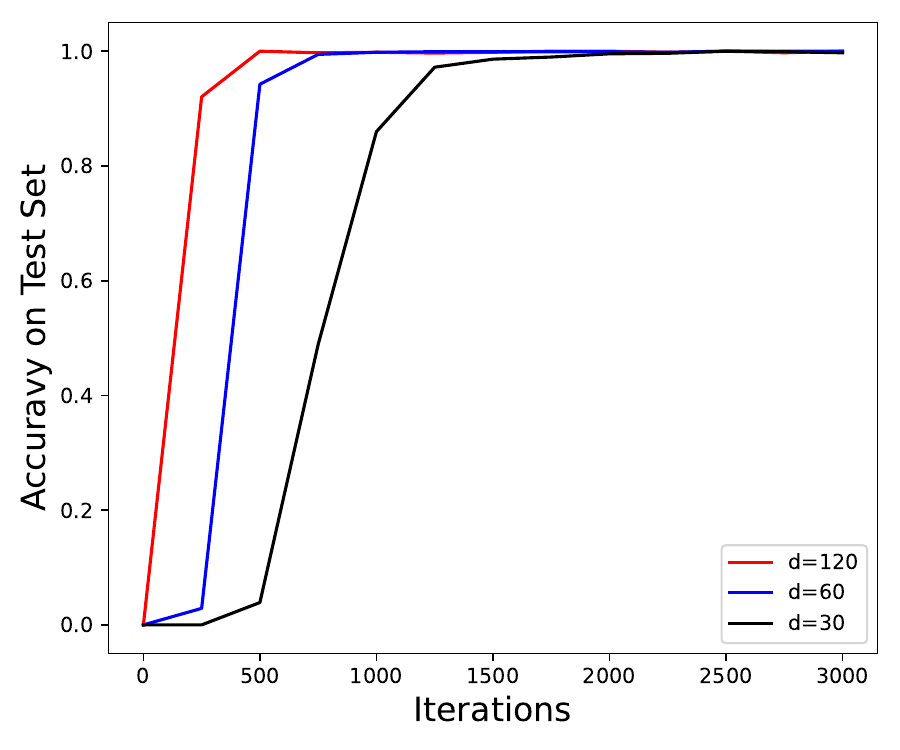}
        \label{fig:BW-accuracy}}
    \end{subfigure}
    \begin{subfigure}[Average Weight Gap]{
        \centering
        \includegraphics[width=0.35\linewidth]{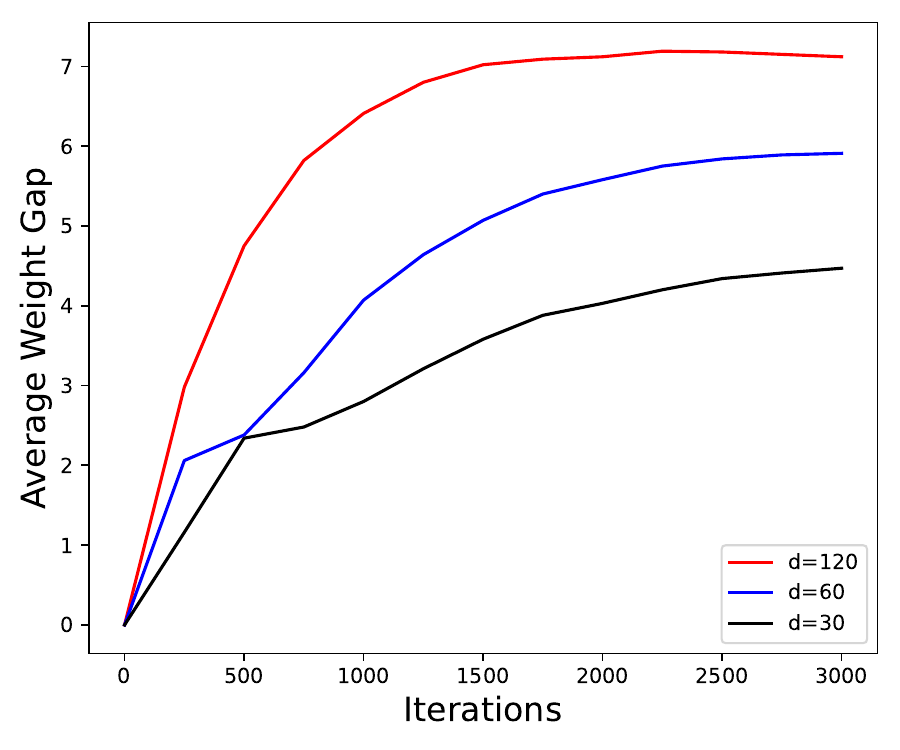}
        \label{fig:BW-weight-gap}}
    \end{subfigure}
    \caption{Accuracy, attention, and adjacency matrix results for the experiment on Blocksworld benchmark.}
    \label{fig:BW}
\end{figure}

\begin{figure}[t]
    \centering
    \begin{subfigure}[Observed Reachability]{
        \centering
        \includegraphics[trim=5 5 5 5, clip, width=0.3\linewidth]{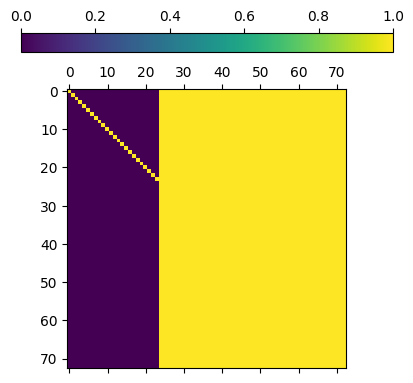}
        \label{fig:BW-reach-real}}
    \end{subfigure}
    \begin{subfigure}[The $\bm{W}^{V'}$ Matrix]{
        \centering
        \includegraphics[trim=40 25 17 40, clip, width=0.3\linewidth]{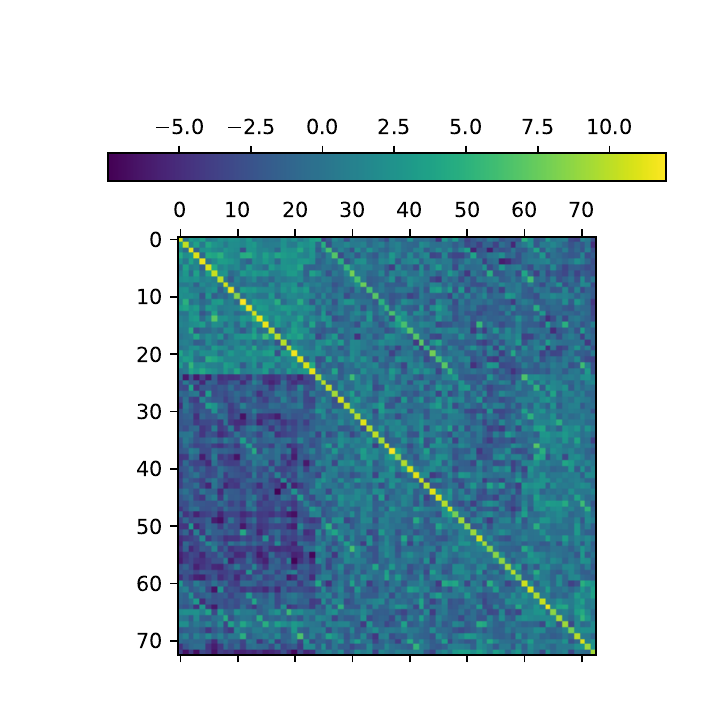}
        \label{fig:BW-reach-learned}}
    \end{subfigure}
    \begin{subfigure}[Average Weight Gap]{
        \centering
        \includegraphics[width=0.32\linewidth]{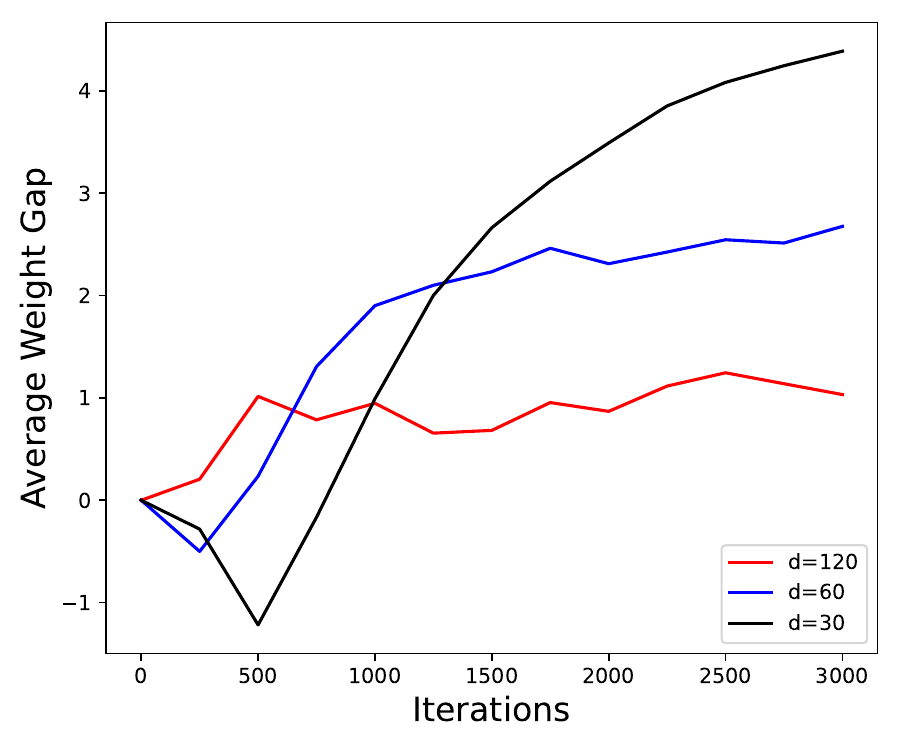}
        \label{fig:BW-reach-weight-gap}}
    \end{subfigure}
    \caption{Experiment for reachability on Blocksworld benchmark.}
    \label{fig:BW-reach}
\end{figure}
\subsection{Results}
We first present the accuracy results during training when using different embedding sizes $d \in \{30, 60, 120\}$. As shown in Figure~\ref{fig:BW-accuracy}, although a smaller embedding size results in a longer time to converge, all models reach an accuracy near $100\%$ at the end of the training. 

Then, we use the same method introduced in Section~\ref{sec:empirical} to visualize the attention map and the $\bm{W}^{M'}$ matrix for the model with $d=120$ after the %whole 
entire iterations. In Figure~\ref{fig:BW-attention}, we can see that when predicting the tokens on the path, almost all the attention weights are on the second token which represents for the target node, demonstrating the capability of the model to learn a correct attention. For adjacency matrix, we find that the $\bm{W}^{M'}$ matrix in Figure~\ref{fig:BW-adjacency-learned} almost perfectly aligns to the real adjacency matrix of $G_{BW}$. And the weight gap (average edge weight minus average non-edge weight) for all models keeps increasing in the training process until convergence, as shown in Figure~\ref{fig:BW-weight-gap}.

%Besides,
In addition, we present the results related to observed reachability matrix in Figure~\ref{fig:BW-reach}. Figure~\ref{fig:BW-reach-real} shows the observed reachability in the training dataset. Although $G_{BW}$ is fully-connected, some reachability are not observed since we request that all training data has no cycle. Specifically, in this case, each of the first $24$ nodes is not observed reachable to any nodes other than itself. To validate whether the Transformer has captured this information, we construct $\bm{W}^{V'}$ matrix through the same method presented in Section~\ref{sec:empirical}. As shown in Figure~\ref{fig:BW-reach-learned}, the first 24 columns of the $\bm{W}^{V'}$ matrix are noticeably darker, which aligns with the observed reachability matrix in Figure~\ref{fig:BW-reach-real}. Furthermore, we plot the gap between the average weight of $\bm{W}^{V'}$ on observed reachability and the average weight of $\bm{W}^{V'}$ on non-observed reachability in Figure~\ref{fig:BW-reach-weight-gap}, and find that this gap keeps increasing for all models.
Since there does not exist any test pairs with degree 2 or more (as defined in Section \ref{sec:empirical}), we do not compare the accuracy between different degrees in Blocksworld.
%\siwei{can we test different degrees?}

In summary, our experimental results on the Blocksworld benchmark confirm our theoretical analyses (Theorem~\ref{Thm_1}) and empirical results on the synthetic data, and they
	at least partially explain the planning capability of the Transformer on the Blocksworld scenario.

%Optionally include supplemental material (complete proofs, additional experiments and plots) in appendix.
%All such materials \textbf{SHOULD be included in the main submission.}

%%%%%%%%%%%%%%%%%%%%%%%%%%%%%%%%%%%%%%%%%%%%%%%%%%%%%%%%%%%%

\end{document}